\documentclass{article}

\usepackage[preprint]{neurips_2025}

\usepackage[utf8]{inputenc} % allow utf-8 input
\usepackage[T1]{fontenc}    % use 8-bit T1 fonts
\usepackage{hyperref}       % hyperlinks
\usepackage{url}            % simple URL typesetting
\usepackage{booktabs}       % professional-quality tables
\usepackage{amsfonts}       % blackboard math symbols
\usepackage{nicefrac}       % compact symbols for 1/2, etc.
\usepackage{microtype}      % microtypography
\usepackage{xcolor}         % colors
\usepackage{bm}
\usepackage{amsmath}
\usepackage{amssymb}
\usepackage{amsthm}
\usepackage{thmtools}
\usepackage{thm-restate}
\usepackage{cleveref}
\usepackage{graphicx}
\usepackage{enumitem}
\usepackage{wrapfig}

\theoremstyle{plain}
\newtheorem{theorem}{Theorem}[section]

\newtheorem{lemma}[theorem]{Lemma}

\newtheorem{corollary}[theorem]{Corollary}

\DeclareMathOperator*{\argmax}{argmax}

\title{Data Shifts Hurt CoT: A Theoretical Study}

  \author{%
Lang Yin \quad Debangshu Banerjee \quad Gagandeep Singh\\
Department of Computer Science\\
University of Illinois Urbana-Champaign\\
\texttt{\{langyin2, db21, ggnds\}@illinois.edu}\\
}
\begin{document}

\maketitle

\begin{abstract}
    Chain of Thought (CoT) has been applied to various large language models (LLMs) and proven to be effective in improving the quality of outputs. In recent studies, transformers are proven to have absolute upper bounds in terms of expressive power, and consequently, they cannot solve many computationally difficult problems. However, empowered by CoT, transformers are proven to be able to solve some difficult problems effectively, such as the $k$-parity problem. Nevertheless, those works rely on two imperative assumptions: (1) identical training and testing distribution, and (2) corruption-free training data with correct reasoning steps. However, in the real world, these assumptions do not always hold. Although the risks of data shifts have caught attention, our work is the first to rigorously study the exact harm caused by such shifts to the best of our knowledge. Focusing on the $k$-parity problem, in this work we investigate the joint impact of two types of data shifts: the distribution shifts and data poisoning, on the quality of trained models obtained by a well-established CoT decomposition. In addition to revealing a surprising phenomenon that CoT leads to worse performance on learning parity than directly generating the prediction, our technical results also give a rigorous and comprehensive explanation of the mechanistic reasons of such impact.
    %Focusing on the $k$-parity problem, we investigate the the following problem: How do distribution shift and wrong CoT steps impact the quality of output? Further, how severe is the deterioration quantitatively? Second, for difficult problems which we must rely on additional reasoning, how much poisoning can a well-designed transformer tolerate while still maintaining satisfactory performance? The answers to those problems answer a more critical question: When does chain of thought degrade the performance instead of improving it?
\end{abstract}

\section{Introduction}

Large language models (LLMs) based on the transformer architecture has achieved tremendous success in the area of artificial intelligence~\citep{attentionisallyouneed}. However, without intermediate guidance or supervision, they do not perform well especially on complex reasoning problems which require rigorous logical steps~\cite{sakarvadia-etal-2023-memory}. Chain of Thought (CoT) has empowered LLMs to a large extent~\citep{CoT2022Wei}, making them much more capable at multi-step reasoning~\citep{nye2021, CoT2022Wei, star2022, lightman2024}, and more effective against hallucinations~\citep{dhuliawala-etal-2024-hallucination}. From a theoretical point of view, CoT has recently been proven to fundamentally improve the power of transformers from a complexity-theoretic perspective~\citep{merrillexpressive, merrill-sabharwal-2023-parallelism, merrill-etal-2022-saturated, liCoT2024}. Several works have applied the CoT mechanism to solve concrete mathematical problems that are hard for primitive models, such as function classes via in-context learning~\citep{li2023incontextfunction, bhattamishra2024} and the $k$-parity problem~\citep{kim2025parity}.

However, all previous works only took care of the case with perfect training data for CoT reasoning. In practice, there can be distribution shifts between training and testing data, as well as some of the CoT reasoning steps used for training can be incorrect. It is an open problem to assess the performance of CoT under such shifts. Even for mathematical problems with a clear structure, this question remains open without a comprehensive answer.

In this paper, we conduct a comprehensive and theoretical study of the impact of training data on CoT for the task of learning parity functions. In one unified \Cref{main}, we characterize a necessary and sufficient condition for CoT success on this task under the joint impact of distribution shifts and data poisoning. The condition gives a decisive criterion on the success of this training method by assessing the parameters of the two types of data shifts concurrently. To the best of our knowledge, this paper is the first work to study this problem from a theoretical perspective.

\paragraph{Motivations.} In this paper, we focus on answering the following major question:
\begin{center}
    Is CoT still effective under data shifts?
\end{center}
To tackle this question, we choose the ``\textbf{generalized}'' $k$-parity problem as a platform, where the term ``\textbf{generalized}'' refers to a broad class of input generating distributions. They will be rigorously defined in Section 2, and they roughly divide the classes of $k$-parity problems into an ``\textit{easy}'' class and ``\textit{hard}'' class; the level of ``\textit{easiness}'' (or ``\textit{hardness}'') can be precisely, quantitatively characterized by the parameter of input generating distributions. At a high level, a non-uniform distribution makes the problem easier by ``leaking'' information on the locations of target bits. Another type of shift is data poisoning in the training data of CoT steps. In this work, we investigate the joint impact of both training distributions and data poisoning in those steps on the performance of CoT for the generalized k-parity problem. The answers revealing this three-way relationship is our main contribution in this work. The major question can further be divided into the following more specific ones, and they are the key topics being investigated in our work:
\begin{enumerate}[noitemsep, nolistsep, leftmargin=*]
    \item Does more information leak help the algorithm to identify the correct positions of relevant bits?
    \item How severe is the impact of data poisoning? We divide this question into two more specific ones:
    \begin{enumerate}
        \item What is the threshold on the level of poisoning that learning can tolerate? 
        \item Is there a specific pattern of corruption that harms learning?
        %\item Can we explain the mechanism of such harm?
    \end{enumerate}
    \item How do both types of shifts affect the training if they concurrently exist?
    \item Can we explain the mechanism of such effect?
\end{enumerate}
All five questions will be answered in Section 4.2.

%The significance of data quality raises a few additional questions: (1) Is there a quantitative bottom line of data quality for successful training; (2) What degrades data quality; and (3) Why does a low data quality cause difficulties for training? The three statements are simple forms, and their more formal version will be stated in the next section.

\paragraph{Choice of the $k$-parity problem.} The $k$-parity problem aims to guess the sign of the product of $k$ selected bits from $\{ -1, 1 \}$ among a large number of bits, and consequently identify the relevant positions involved in the product. We select it for two main reasons.
\begin{enumerate}[leftmargin=*]
    \item The major theme of our paper is about the impact of distribution shifts and data poisoning. Any shift on binary inputs for the $k$-parity problem can be easily, objectively quantified. Similarly, at every step and every position, an entry has only one correct value given a number of bits, so anything other than the true value is poisoning.
    \item The quality of a parity predictor can be objectively assessed as the correctness of the output has an absolute criterion.
\end{enumerate}

\paragraph{Contributions.} Our contributions are summarized as follows.
\begin{enumerate}[noitemsep, nolistsep, leftmargin=*]
    \item In \Cref{easy-input-output}, first show that the imbalanced $k$-parity problem is ``easy", meaning it indeed can be efficiently solved by a one-layer transformer in one gradient update without CoT, and the optimization landscape is benign. %Consequently, CoT provides limited additional benefits on this problem, and if poisoning is stronger than a threshold, CoT can harm the performance even for this simpler problem.
    \item Next, we reveal the joint impact of the distribution shift and data poisoning on the performance of the predictor trained by the successful CoT decomposition of the $k$-parity problem introduced in~\citep{kim2025parity}. This three-way relationship is compressed into one statement in \Cref{main}, and characterizes a necessary and sufficient condition on the amount of distribution shift and data poisoning to ensure successful training. This result has several implications.
    \begin{itemize}[noitemsep, nolistsep, leftmargin=*]
        \item The tolerance of corrupted CoT training samples is only $O(1/k)$, making the learning vulnerable against data poisoning.
        \item Surprisingly, distribution shift always hurts: A higher degree of shift always leads to worse training performance. In our setting, recall that non-uniform distributions leak information on the location of target bits, so intuitively it should help the predictor to learn. However, our result suggests the opposite, and \Cref{rho=0-failure} eliminates the possibility of successful learning when the locations are exposed to the maximum extent.
         %\item In our result, we provide a comprehensive analysis on the quantity of error, and reveal the particular structure of poisoning that can work adversely for this training paradigm. Nevertheless, the training algorithm is tolerant to an extent, and our result concludes an equivalent condition of successful training despite poisoning: As long as the input dimension is large enough and the poisoning exist in a particular form and accumulate below a threshold, the output predictor accurately identifies the relevant elements for the target parity function.
         %\item The impact of training distribution is more surprising. If the distribution is ``imbalanced'' as described above, it simplifies the $k$-parity problem by leaking information about the positions of relevant bits. Intuitively, if the training is empowered with CoT, the algorithm shall more effectively learn those positions and output a more accurate predictor. It is reasonable to assume such a predictor shall be robust against the uniform testing distribution because all it needs is to identify the positions of relevant bits. However, our result shows that, the more information the training distribution leaks, the harder the training is. At the extreme case, if the imbalanced distribution leaks as much information as it can, this CoT decomposition would not work at all.
    \end{itemize}
    %\item For balanced $k$-parity, a provably more difficult problem, the harm of high poisoning preserves as expected. Nevertheless, we show that if the poisoning is below a threshold, then CoT still functions well as a necessary component to solve the problem. \lang{Citations of theorems will be added after finishing the respective sections.}
\end{enumerate}

\section{Related works}

\paragraph{Empowerment of CoT.} Starting from 2022, a line of work investigates the expressive powers of transformers from a complexity theoretic perspective. The most recent and comprehensive works include~\citep{merrillexpressive} and~\citep{liCoT2024}. The first paper provides a comprehensive complexity-theoretic relationship between CoT steps and computational power. Almost concurrently, the second paper proves tighter upper bounds for constant-precision transformers. Together, these two works confirm that, with sufficient (polynomial) CoT steps, transformers break their original upper bound in computation and can compute any problem in \texttt{P/poly}. Beyond theoretical soundness of those works, more concrete implementations on CoT are also being studied, including concrete training paradigms~\citep{li2024nonlinear} and interactions with inference-time search and reinforcement learning fine-tuning~\citep{kim2025Metastable}.

\paragraph{Empirical discoveries of limitations on CoT.} Despite both theoretical and empirical successes, recent empirical works also revealed that sometimes CoT may worsen the performance~\citep{shaikh2023, kambhampati24a}. The CoT mechanism has been shown to improve performance mainly on mathematical and logical tasks, but less so for other tasks~\citep{sprague2025tocotornot}. For tasks where thinking can make human performance worse, the harm caused by overthinking also holds for models with CoT~\citep{liu2024mindyourstep}. It was also found recently that transformers can still solve problems with meaningless filler CoT tokens~\citep{pfau2024COLM}.

\paragraph{Parity and LLMs.} It has been shown that if, with a positive probability, the relevant bits are uniformly $1$ or $-1$, then this ``\textbf{imbalanced}'' $k$-parity problem can be solved by an one-layer neural network~\cite{daniely2020}. However, if all inputs are uniformly generated, then this uniform $k$-parity problem has been proven not to be solvable by any input-output learning algorithm based on gradient updates~\citep{shalev2017, shamir2018}. Recently, thanks to developments of CoT, several works have made significant progresses on solving uniform $k$-parity with task decompositions as CoT steps. Success has been achieved for recurrent neural networks in~\cite{wiessub2023}, and they designed a task decomposition of $k$-parity into $k-1$ structured steps. Afterwards, \cite{kim2025parity} extended their results to autoregressive transformers. 

\section{Problem setup}

\paragraph{Notation.} We write $[n] := \{1, \cdots, n\}$ for any integer $n$. The multi-linear inner product of vectors $\bm{z}_1, \ldots, \bm{z}_r \in \mathbb{R}^n$ for any $r \in \mathbb{N}$ is denoted as $\langle \bm{z}_1, \cdots, \bm{z}_r \rangle := \sum_{i=1}^{n} z_{1,i} \cdots z_{r, i}$. In particular, $\langle \bm{z} \rangle = \bm{z}^{\top} \bm{1}_n$ and $\langle \bm{z}_1, \bm{z}_2 \rangle = \bm{z}_1^{\top} \bm{z}_2$. The transformer will be denoted by a function $\text{TF}(\cdot)$ Unless specified, each binary vector $\bm{x}$ represents the ground truth, and $\hat{\bm{x}}$ is the generated vector by the transformer. The definition of data poisoning or data corruption will be formally presented later, but if $\bm{x}$ is injected with poisoning, then it is denoted by $\tilde{\bm{x}}$.

\subsection{The parity problem}

Let $d \geq k \geq 2$ be integers, and $P$ be an arbitrary subset of $[d]$ with $k$ elements. In this paper, we study the $k$-parity problem, where the output of the target parity function is $y = \prod_{j \in P} x_j$, so the function value called \textit{parity}, entirely depends on the coordinates at the locations determined by $P$. Given $n$ samples $(\bm{x}^i, y^i)_{i \in [n]}$, our goal is to predict the parity of any test input from $\{ \pm 1 \}^d$. We assume $k = \Theta(d)$.

It is known that, if all inputs $\bm{x}^i$ of dimension $d$ are uniformly generated, then this ``\textbf{uniform}" $k$-parity problem is fundamentally difficult and cannot be solved in polynomial time by any finite-precision gradient-based algorithms~\citep{wiessub2023}. Recently, it was proven that the uniform $k$-parity problem can be solved by transformers with $\log k$ reasoning steps~\citep{kim2025parity}.

On the other hand, for a particular kind of imbalanced distribution on the input bits for training, the $k$-parity problem is proven to be solvable by neural networks~\citep{daniely2020}. The ``imbalanced" distribution is defined as the following: For any number $\rho \in [0,1]$ and a subset $P$ of $[d]$, the distribution $\mathcal{D}^P_{\rho}$ is a distribution on the input bits such that
\begin{itemize}[leftmargin=*]
    \item With probability $\rho$, all $d$ bits are uniformly generated.
    \item With probability $1-\rho$, all bits in $[d] \setminus P$ are uniformly generated, but the bits in $P$ are all $1$ with probability $1/2$, and all $-1$ also with probability $1/2$.
\end{itemize}
If $\rho = 1$, $D^P_{\rho} = D^P_1$ reduces to the uniform distribution. Intuitively, any distribution $D^P_{\rho}$ with $\rho < 1$ leaks information for the relevant bits and consequently makes the parity problem easier. As we will show later in Section 4.1, this imbalanced $k$-parity can also be solved by a one-layer transformer. If the value of $\rho$ is not specified, we categorize such a problem a ``\textbf{generalized}'' $k$-parity problem. %, and surprisingly, CoT reasoning may work adversely against this problem.

\subsection{Chain of Thought (CoT)}

%In this section, we briefly present the intuition of chain of thought (CoT) at a high level and then discuss more concrete settings and the motivation to study robustness for our task.

Designed and first implemented in 2022~\citep{CoT2022Wei}, the \textbf{chain of thought (CoT)} approach has been a powerful technique to improve large language model's performance for reasoning tasks. At high level, CoT inquiries ask the language model to generate intermediate steps instead of outputting a final answer directly. The intermediate steps are obtained by applying the transformer repetitively using earlier intermediate tokens, and they can be regarded as a ``reasoning chain''.% Thanks to the fact that the $k$-parity problem is a well-stated mathematical problem, the intermediate steps in our case have more concrete forms.

\paragraph{Task decomposition as CoT.} As in~\citep{wiessub2023} and~\citep{kim2025parity}, we assume $k=2^v$ for simplicity, where $v \in \mathbb{N}$. The chain of thought protocol decomposes the $k$-parity problem as a sequence of $2$-parity problems. Visually, this is expressed as a complete binary tree of height $v$ and $2k-1$ nodes. The lowest level on this tree contains $k$ nodes from $P$, and the remaining $d-k$ irrelevant bits are isolated vertices and not part of the tree. All remaining nodes represent reasoning steps, and are labeled as $x_{d+1}, \ldots, x_{d+k-1}$. The next level above contains $k/2$ nodes, then $k/4$, and so on until at level $v$, the unique node $x_{d+k-1}$ is the final prediction of the parity value. For each $d < m < d+k-1$, $x_m$ must have exactly two children and they are denoted by $c_1[m]$ and $c_2[m]$. At the same time, it must have exactly one parent and it is denoted by $p[m]$. The height of the tree is denoted as $h[m]$, the length of the longest path in the graph. All $d$ nodes corresponding to $d$ inputs are located at level zero; the height of any other node is difference between $h[m]$ and the number of edges from itself to the root.

\paragraph{Feedforward layer.} The feedforward layer carries a fixed link function $\phi: [-1,1] \rightarrow [-1,1]$, applied element-wise. To exploit the decomposition of our task into 2-parities, we choose $\phi$ such that $\phi(0) = -1$, $\phi( \pm 1) = 1$ so that sums are converted into parities, i.e. $\phi(\frac{a+b}{2}) = ab$ for $a, b \in \{ \pm 1 \}$. Moreover, we require symmetry of $\phi$, and that $\phi'(0) = 0$. Specifically, we choose the following function.
\begin{equation}
    \phi(x) = \begin{cases}
    d^3 x^2 + d^{-3} -1, \quad x \in (-d^{-3}, d^{-3});\\
    2|x| - 1, \quad \quad \quad \quad \text{otherwise}.
    \end{cases}
\end{equation}
The choice of this function ensures that $\phi$ is differentiable everywhere on $[-1,1]$. There is a discrepancy on $\phi(0) = d^{-3}-1$ instead of $-1$, but the gap will be bounded as perturbations and approach to zero as $d$ becomes large.

\begin{figure}
    \centering
    \includegraphics[scale=0.7]{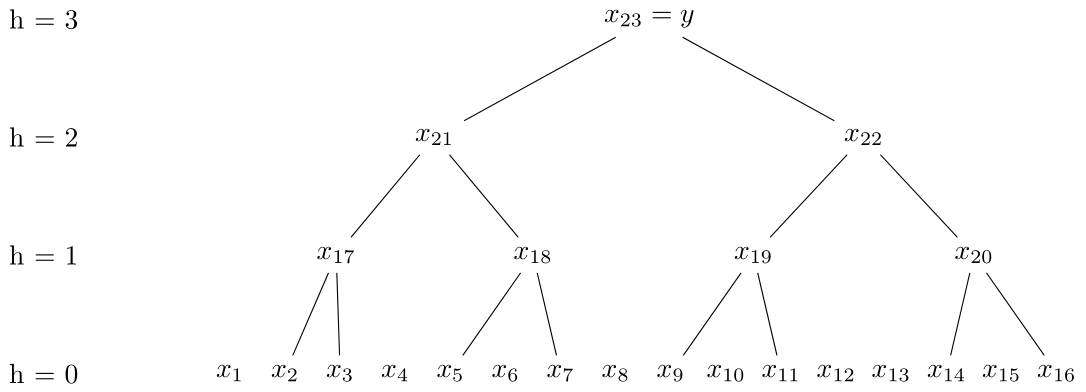}
    \caption{A hierarchical decomposition of an 8-parity problem for $d=16$. Over here, $x_{17} = x_2 x_3$, so $c_1[17] = 2$, $c_2[17] = 3$, $p[17]=21$, and $h[17]=1$.}
    \label{fig:cot-tree}
\end{figure}

\paragraph{Specific CoT process.} For our case on $k$-parity, all input bits $\bm{x}_1, \ldots, \bm{x}_d$ are fixed, and the positions of later steps before generation are null. The first intermediate token, $\hat{\bm{x}}_{d+1}$, is generated next and staying the same value through the entire remaining process. Similarly, the next token is generated by $\hat{\bm{x}}_{d+2} = \text{TF}^{(2)} (\bm{x}_1, \cdots, \bm{x}_d, \hat{\bm{x}}_{d+1}; \bf{W})$, where $\bf{W}$ is the transformer weights. Finally, the final prediction is computed by repeating the computation for $k-1$ times:
\begin{equation}
    \bm{y} = \text{TF}^{(k-1)} (\bm{x}_1, \ldots, \bm{x}_d, \hat{\bm{x}}_{d+1}, \ldots, \hat{\bm{x}}_{d+k-2}; \bf{W})
\end{equation}

\paragraph{Teacher forcing.} For CoT implementations, we utilize \textit{teacher forcing} in our training process. Teacher forcing is a form of process supervision, where in addition to the final prediction, ground truth labels for CoT steps are provided during training. Consequently, the accuracy of each CoT step can be measured. Given $n$ samples and model weights $\mathbf{W}$, the total loss takes every position $d+1 \leq m \leq d+k-1$ into account and is defined as
\begin{equation}
    L(\mathbf{W}) = \frac{1}{2n} \sum_{m=d+1}^{d+k-1} \lvert \lvert \hat{\bm{x}}_m - \bm{x}_m \rvert \rvert^2 = \frac{1}{2n} \sum_{m=d+1}^{d+k-1} \lvert \lvert \phi(\hat{\bm{z}}_m) - \bm{x}_m \rvert \rvert^2, \quad \bm{z}_m = \sum_{j=1}^{m-1} \sigma_j (\bm{w}_m) \bm{x}_j,
\end{equation}
where $\sigma_j(\bm{w}_j)$ are softmax attention scores.

\paragraph{CoT data corruption.} Intuitively, if all CoT steps are correct for training, then indeed the final predictor will accurately reflect the true target function thanks to correct decomposition. However, correctness of CoT steps relies on correct ``ground truth'' labels during the training steps. If a high amount of such tokens are false, because of either oversight or malicious attacks, then naturally, one may infer that the quality of those intermediate steps may deteriorate. In our case, the ground truth labels are either $1$ or $-1$, and corruption refers to flipping the signs of some inputs.

%\lang{The next subsection should be about the training, teacher forcing, and what ``poisoning'' means in our context.}

\section{Theoretical results}

\subsection{Imbalance is easier than uniformity for transformers without CoT}

In this section, we assume $\rho < 1$ so information on relevant bits in $P$ are leaked. We assume $\rho$ is never too small nor too large, so $\rho = \Theta(1)$. The goal of this section is to show that the imbalanced problem is indeed solvable by a simple, one-layer transformer with a softmax attention layer. \Cref{easy-input-output} is the specific statement of this result. The proof will be presented in the appendix.

\begin{theorem} \label{easy-input-output}
    Given $n$ samples where $n = \Omega (d^{2 + \epsilon})$, with probability at least $1 - \exp(d^{\epsilon/2})$, a learning rate $\eta = \Theta (d^{\epsilon})$ and all-zero initializations, the predictor $\hat{y}$ after one-step update satisfy $| \hat{y} - y | \leq O(d^{-1+\epsilon})$ for any given input $\bm{x} \in \{ \pm 1 \}^d$ and $y = \prod_{r \in P} x_r$.
\end{theorem}

\begin{proof}[Proof sketch]
    %The proof of this theorem utilizes the variant of the tree described in the previous section. As \Cref{fig:flat-tree} shows, this tree has only height one, where the bottom layer contains $k$ leaves representing all relevant bits, and they share a same parent $x_{d+1}$, which is the final predicted value.
    
    The proof involves explicitly computing the gradient with respect to each weight $w_{j,m}$, and utilizing the large differences on the gradients between relevant and irrelevant bits (whether $j \in P$ or not). Because the softmax scores are identical at initialization, we can compute the interaction terms among the tokens $\bm{x}_1, \ldots, \bm{x}_d, \hat{\bm{x}}_{d+1}$, where $\hat{\bm{x}}_{d+1}$ is the vector of predictions.
    
    Here is an example of interaction: $\left \langle \bm{x}_{d+1}, \hat{\bm{z}}_m, \hat{\bm{z}}_m \right \rangle = \left \langle \bm{x}_{d+1}, \bm{x}_{\alpha}, \bm{x}_{\beta} \right \rangle/d^2$. If $\alpha, \beta \in P$, then for each data sample, the parity $x_{d+1} x_{\alpha} x_{\beta} = \prod_{i \in P \setminus \{ \alpha, \beta \}} x_i$ is a random variable with expectation $1 - \rho > 0$ if $\rho < 1$. If all bits are uniformly generated, i.e. $\rho = 1$, then this variable has mean zero. Since the training distribution $\mathcal{D}^P_{\rho}$ is imbalanced, the variable has a positive mean. But if at least one of $\alpha$ and $\beta$ is not in $P$, such variables are bounded with in a small value. With a sufficiently large $d$ and $n$, the computation leads to positive weight updates for relevant bits, and negative updates for others. Such a huge gap allows an attention layer to identify relevant bits and make correct predictions.
\end{proof}

\subsection{Applying CoT for the generalized problem}

In this section, the value of $\rho \in [0,1]$ is not restricted, and we focus on the impact of CoT decomposition on solving the \textbf{generalized} $k$-parity problem under two types of nuances: (1) the information leaked (regulated by $\rho$), and (2) the level of data poisoning, defined in Section 3.

Our main result \Cref{main} characterizes an equivalent condition on successful training with this CoT decomposition with respect to (1) the distribution shift $\rho$ and (2) the quantity and distribution of data poisoning. Before stating the theorem, we first define a few ingredients for the characterization.

\paragraph{Set of poisoning.} We first define a few notations for quantifying the data poisoning on different positions and their interactions. Given $n$ data samples of dimension $d+k-1$ including the CoT steps, for each node $i \in \{ d+1, \ldots, d+k-1 \}$, let $U_i \subseteq [n]$ be the set of indices with corrupted samples. For any two nodes $a, b \in \{ d+1, \ldots, d+k-1 \}$, define the set $I_{a,b} = U_a \cap U_b$. For any $i \in I_{a,b}$, coordinates $a$ and $b$ in the training sample $\bm{x}^i$ are both flipped by multiplying $-1$, so the effect of poisoning cancels out. For three nodes $a, b, c \in \{ d+1, \ldots, d+k-1 \}$, we define the following set:
\begin{equation*}
    U_{a, b, c} = \left \{ \left( U_a \cup U_b \cup U_c \right) \setminus \left( I_{a,b} \cup I_{a,c} \cup I_{b,c} \right) \right \} \cup \left( U_a \cap U_b \cap U_c \right).
\end{equation*}
For two nodes $a$ and $b$, the set $U_{a, b} = (U_a \cup U_b) \setminus (U_a \cap U_b)$ is defined with the same rationale. See \Cref{fig:venn} and \ref{fig:poisoning} for visualizations. We will denote $q_{a,b,c} = | U_{a,b,c} |$ for the size of this set.

% \begin{wrapfigure}{r}{0.45\textwidth}
%     \centering
%     \includegraphics[width=0.42\textwidth]{figures/Venn.pdf}
%     \caption{Given three sets $U_a$, $U_b$, and $U_c$, the set $U_{a,b,c}$ is the union of all three sets excluding elements that belong to exactly one intersection of two sets.}
%     \label{fig:venn}
% \end{wrapfigure}

% \begin{wrapfigure}{l}{0.45\textwidth}
%     \centering
%     \includegraphics[width=0.42\textwidth]{figures/three-way poisoning.png}
%     \caption{Each of the three axes represents the samples of nodes $a$, $b$, and $c$. Red segments indicate flipped samples. Flips from $0.2n$ to $0.3n$ and from $0.5n$ to $0.7n$ cancel out due to flipping twice. Other red segments are harmful, occurring at one node only ($0.1n \to 0.2n$ and $0.7n \to 0.9n$) or at all three nodes ($0.3n \to 0.5n$).}
%     \label{fig:poisoning}
% \end{wrapfigure}
\begin{figure}[ht]
    \centering
    \begin{minipage}[t]{0.48\textwidth}
        \centering
        \includegraphics[width=0.55\textwidth]{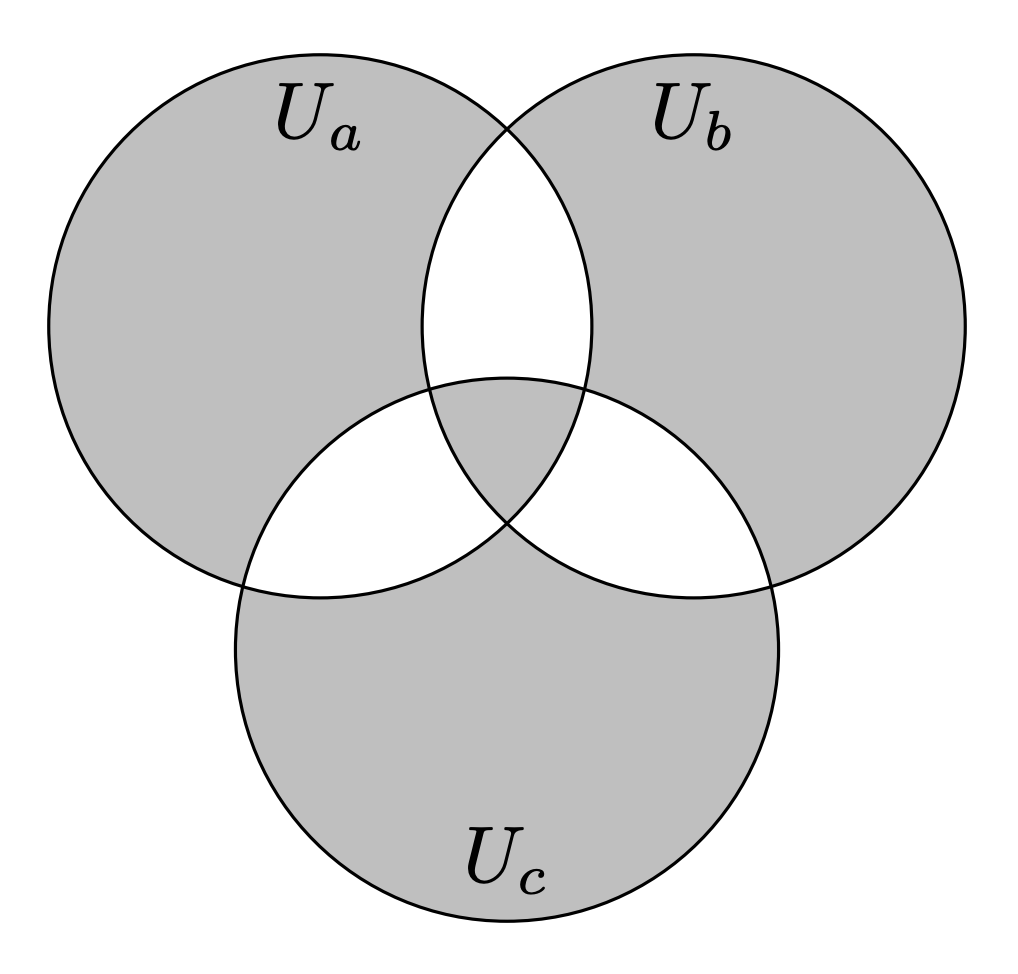}
        \caption{Given three sets $U_a$, $U_b$, and $U_c$, the set $U_{a,b,c}$ is the union of all three sets excluding elements that belong to exactly one intersection of two sets. The shaded region represents the impactful corruption. Corruption in the white region is cancelled out.}
        \label{fig:venn}
    \end{minipage}%
    \hfill
    \begin{minipage}[t]{0.48\textwidth}
        \centering
        \includegraphics[width=\textwidth]{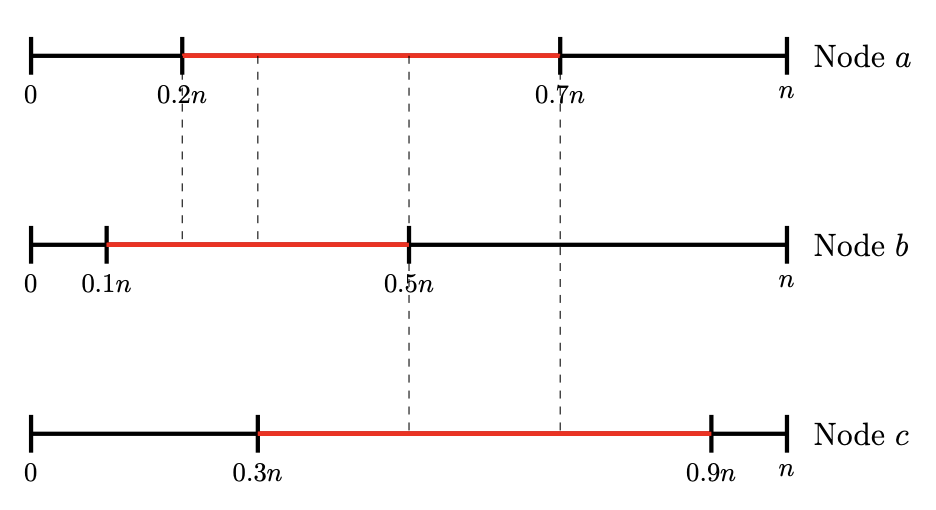}
        \caption{Each axis represents samples of nodes $a$, $b$, and $c$. Red segments show flipped samples. Flips from $0.2n$–$0.3n$ and $0.5n$–$0.7n$ cancel out due to two flips. Other red segments are harmful: flipped at one node ($0.1n$–$0.2n$, $0.7n$–$0.9n$) or all three nodes ($0.3n$–$0.5n$).}
        \label{fig:poisoning}
    \end{minipage}
\end{figure}

\paragraph{Quantities for poisoning characterization.} Exact quantities of gradient updates $\partial L / \partial w_{j,m}$ for every $1 \leq j < m$ depend on the indices of $j$ and $m$. Those quantities are necessary to characterize the poisoning and state our main result \Cref{main}, but they have considerably long expressions.

We define the following three functions. They are ingredients to compute the differences for gradient updates between correct and incorrect nodes. The final expressions of gradient updates have two major terms, the first term distinguished on (1) CoT steps of height one $\left( d \leq m < d+\frac{k}{2} \right)d$ or higher $\left( d + \frac{k}{2} < m \leq d+k-1 \right)$, and (2) location of $j$: either $h[j] = 0$ or $h[j] \geq 1$. The second term only distinguishes on the latter.

For any $1 \leq j \leq m$, the name $G_{h[m]=1}$ denotes the signal difference for the CoT steps of height one with distinct levels of $j$, and $G_{h[m]>1}$ is defined similarly for later CoT steps. The $S(m)$ is the coefficient of the signal difference given by the second term between $j$ on the lowest level and on higher levels. Their expressions are formally stated below.
%\begin{equation}
%    S(\rho, d, k, m) = - \frac{2(k-1)}{(m-1)^2} (1-\rho) + \sum_{d < \alpha < m, \alpha \neq j} \frac{2(1-2q_{\alpha})}{(m-1)^2} (1-\rho) + \sum_{\alpha = d+1}^{m-1} \frac{8k(1-2q_{\alpha})}{(m-1)^2} (1-\rho) + \sum_{\alpha, \beta = d+1}^{m-1} \frac{4(1-2q_{\alpha, \beta, j})}{(m-1)^2} (1-\rho)
%\end{equation}
\begin{small}
\begin{equation} \label{eq:G(h[m]=1)}
    G_{h[m] = 1} (m, j, \rho) = - \frac{2(1-2q_m)}{(m-1)^2} - \frac{2(k-1)(1-2q_m)}{(m-1)^2} (1-\rho) + \sum_{\alpha=d+1}^{m-1} \frac{2(1-2q_{m,\alpha,j})}{(m-1)^2} (1-\rho).
\end{equation}
\begin{equation} \label{eq:G(h[m]>1)}
    G_{h[m] > 1} (m, j, \rho) = - \frac{2(1-2q_{m, c_1[m], c_2[m]})}{(m-1)^2} - \sum_{d < \alpha < m, \alpha \neq c'[m]} \frac{2(1-2q_{m, c[m], j})}{(m-1)^2} (1-\rho) + \frac{2k(1-2q_m)}{(m-1)^2} (1-\rho).
\end{equation}
\begin{equation} \label{eq:term-2-diff}
    S(m, j) = - \frac{2(k-1)}{(m-1)^2} + \sum_{d < \alpha < m, \alpha \neq j} \frac{2(1-2q_{\alpha})}{(m-1)^2} + \sum_{\alpha = d+1}^{m-1} \frac{8k(1-2q_{\alpha})}{(m-1)^3} -  \sum_{\alpha, \beta = d+1}^{m-1} \frac{4(1-2q_{\alpha, \beta, j})}{(m-1)^3} - \sum_{\alpha, \beta \in P} \frac{4(1-2q_m)}{(m-1)^3}.
\end{equation}
\end{small}
Note that for steps $m$ with $h[m]=1$, it always holds that $q_m = q_{m,a,b}$ if $a, b \in [d]$ because the inputs have no poisoning, so all flips in step $m$ are harmful for computing $\langle \bm{x}_m, \bm{x}_a, \bm{x}_b \rangle$.

We now state our main theoretical result in this work. We answer Question 3 by providing a comprehensive statement on the impact of both distribution shift $\rho$ and the poisoning structure among the CoT steps on the final performance of the predictor within the selected CoT decomposition. Such impact leads to an equivalent (both necessary and sufficient) condition on success of training.

\begin{theorem} \label{main}
    Let $n = \Omega (d^{2 + \epsilon})$ and $\mu > -2 - \epsilon/4$ for $\epsilon > 0$. Suppose $d$ is sufficiently large. With softmax attention and all-zero initializations on weights, a transformer with the prescribed CoT mechanism can solve the uniform parity problem with an error rate converging to zero as $d \to \infty$ if and only if all the following conditions hold.
    \begin{itemize} [noitemsep, nolistsep, leftmargin=*]
        \item For every $m \in \left \{ d+1, \cdots, d + \frac{k}{2} \right \}$, the following inequality satisfies:
        \begin{equation} \label{condition-one-height}
            B_m = \max_{d < j < m} \left\{ -\frac{2\rho (1-2q_m)}{(m-1)^2}, \quad G_{h[m]=1} (m, j, \rho) + (1-\rho) S(m, j) \right \} < -O(d^{\mu}).
        \end{equation}
        \item For every $m \in \left \{ d + \frac{k}{2} + 1, \cdots, m-1 \right \}$, the following inequality satisfies:
        \begin{equation} \label{condition-higher-height}
            B_m = \max_{d < j < m} \left\{ -\frac{2\rho (1-2q_{m, c_1[m], c_2[m]})}{(m-1)^2}, \quad G_{h[m] > 1} (m, j, \rho) - (1-\rho) S(m, j) \right \} < -O(d^{\mu}).
        \end{equation}
    \end{itemize}
    In particular, if the conditions above hold for every $m$, then let $B = \displaystyle \max_{d < m <  d+k-1} B_m$, for every input $\bm{x} \in \{ \pm 1 \}^d$, the true prediction $y$ and the prediction $\hat{y}$ by the trained predictor after one step update with learning rate $\eta = \Theta(d^{-\mu})$ satisfy $| \hat{y} - y | \leq O(d^{-B(\mu-2-\epsilon/4)})$ with probability at least $1 - \exp(d^{\epsilon/2})$. If the conditions fail for at least one CoT step, then $\lim_{d \to \infty} \mathbb{E} \left[ | \hat{y} - y | \right] = \Omega(1)$.
    %$\left \lvert U_{m, \alpha, \beta} \right \vert < n/2$. In particular, after the first update, if $\left \lvert U_{r, \alpha, \beta} \right \vert /n = q < 1/2$, then $\lvert \lvert y - \hat{y} \rvert \rvert \leq O(d^{-\rho(1-2q)\epsilon/8})$ with probability $1 - \exp \left( -d^{\epsilon/2} \right)$. Otherwise, $\lim_{d \to \infty} \mathbb{E}_{D_{\rho}} \left[ y - \hat{y} \right] = 1$.
\end{theorem}

%The statement answers Question 3 in Section 1 since the condition

%The statement answers the Questions 1, 2 and 3 in Section 1. It claims that both the training distribution and the size of $U_{a,b,c}$ (i.e. the level of corruption for parent-children triples) are critical for the quality of the final output. Before presenting the more technical proof sketch, we first explain some critical implications.

\paragraph{Vulnerability against poisoning.}
The theorem provides a rigid equivalent condition for the algorithm to succeed. Even for one intermediate step $m \in \{ d+1, \cdots, d+k-1 \}$, if $q_{m, c_1[m], c_2[m]} \geq 0.5$, then the condition in \Cref{condition-one-height} or \ref{condition-higher-height} no longer holds, and consequently the training algorithm would not succeed. Recall that for this task, there are $(k-1)n$ values in the training data set for all intermediate steps, but $0.5n$ flips are sufficient to fail the training. Conclusively, this algorithm has a low poisoning tolerance of $\frac{1}{2(k-1)}$, and this threshold approaches to zero as $k$ become large. This analysis answers Questions 2.(a) and 2.(b) in Section 1.

Regarding distributions shift, \Cref{main} leads to an immediate corollary, which reveals a seemingly paradoxical conclusion.

\begin{corollary} [Maximum leakage of information]
    If $\rho = 0$, i.e. all non-relevant bits are still uniformly generated but all relevant bits are either all $-1$ or $1$, then the training always fails with this CoT decomposition.
\end{corollary}

\begin{proof}
    If $\rho = 0$, then $B_m = 0$ for any intermediate node $m$.
\end{proof}

\paragraph{Explanation of the ``paradox''.} If $\rho = 0$, the locations of relevant bits are exposed to the maximum extent, so intuitively, the CoT protocol should be able to solve it even more efficiently than the case when $\rho = 1$. However, within this CoT design, this case leads to an immediate, absolute failure. The exact reason will be briefly outlined in the proof sketch and comprehensively presented in the full proof. At a high level, when $\rho = 0$, the gradient update extracts identical information from \textit{correct} nodes (children) and \textit{incorrect} nodes (non-children) during the CoT steps. However, identifying the location of the children is essential to ensure that the final output is indeed the multiplication of the relevant bits. If errors on this step exist, some relevant bits are multiplied multiple times and therefore the output will be different from the truth.

\paragraph{Ever-present harm of distribution shift.} The impact of $\rho$ is concrete even if $\rho > 0$. Recall that for the uniform case where $\rho = 1$, we have
\begin{equation}
   \begin{cases}
       G_{h[m]=1} (m, 1) = - \frac{2(1-2q_m)}{(m-1)^2} = - \frac{2(1-2q_{m, c_1[m], c_2[m]})}{(m-1)^2}, \\
       G_{h[m]>1} (m, 1) = - \frac{2(1-2q_{m, c_1[m], c_2[m]})}{(m-1)^2}.
   \end{cases}
\end{equation}
Thus, $B_m = - 2(1-2q_{m, c_1[m], c_2[m]})/(m-1)^2$ for any $m$ if $\rho=1$. Let $m' = \argmax_{m} B_m$, then if $\rho < 1$, denote the values computed in \Cref{condition-one-height} and \ref{condition-higher-height} as $\{ B'_m \}_{m=d+1}^{d+k-1}$, observe that $B_{m'} > B_m$. Hence, $B' = \max_m B'_m > B'$ and the convergence rate slows down for every $\rho < 1$. Furthermore, clearly a lower $\rho$ leads to a higher $B_m$. So, we can conclude that distribution shift always damages the training if it exists, and low shift is always better than high shift. This answers Question 1.

\begin{corollary} [Simple characterization without distribution shift] \label{rho=0-failure}
    If the training and testing distribution are identical, i.e. $\rho = 1$, then the CoT decomposition succeeds if and only if $q_{m, c_1[m], c_2[m]} \leq 0.5 - O(d^{\mu})$ for every $m \in \{ d+1, \ldots, d+k-1 \}$.
\end{corollary}

\begin{proof}
    If $\rho  = 1$, then clearly
    \begin{equation}
        G_{h[m] = 1} + (1-\rho)S = G_{h[m] > 1} - (1-\rho)S = - \frac{2(1-2q_{m, c_1[m], c_2[m]})}{(m-1)^2}.
    \end{equation}
    Observe that if $h[m] = 1$, i.e. $m \in \left \{ d+1, \cdots, d + \frac{k}{2} \right \}$, then $q_m = q_{m, c_1[m], c_2[m]}$ because there is no poisoning in the inputs.
\end{proof}

%Before presenting the more technical proof sketch, we first explain the intuition on the statement, and the ``counter-intuition'' on the relationship between $\rho$ and the convergence performance. Both phenomena are due to the quantities of the gradient updates under specific $\rho$ and $q$, which will be discussed in the proof sketch and more carefully analyzed in the proof. At high level, for each CoT step represented as a node, we expect the gradient updates on the correct positions to dominate those for the other nodes. A high corruption level reduces the correct update and therefore weakens such domination, and this phenomenon is intuitive because of the nature of poisoning.

%However, the role of imbalance on training data is counter-intuitive: A lower $\rho$ implies more information leak on the relevant bits, and intuitively this should help the predictor identify the relevant bits. It is then natural to assume that a training protocol empowered by CoT shall perform better on this easier problem than the more difficult uniform problem. However, a lower $\rho$ would slow down convergence as shown in the error rate. As we will see in the proof, information leak leads to non-trivial gradient updates on incorrect positions and makes the predictor harder to distinguish relevant and irrelevant bits.

We have answered all of Questions 1, 2 and 3 in the introduction. We conclude this section by a proof sketch of \Cref{main}, which summarizes the technical analysis that answers Question 4. The entire proof will be presented in the appendix.

\begin{proof}[Proof sketch of \Cref{main}]
    %For both directions, we first compute the gradients for every intermediate step. One critical difference is, for any $d+1 \leq m \leq d+k-1$, the parity $\langle x_m, x_{c_1[m]}, x_{c_2[m]} \rangle$ is always one if the labels are correct.% Clearly, if $x_m$ is wrongly flipped, then the value is always $-1$.
    
    For the \textit{if} direction, we directly compute the gradients, and found that the two quantities in \Cref{condition-one-height} and \ref{condition-higher-height} are gradient differences between correct and incorrect nodes. A low enough value of $B_m$ for every $m$ ensures the two correct nodes have larger weights than incorrect nodes, and the gap must be large enough for the attention layer to distinguish correct and incorrect nodes as $d$ becomes large. As a result, the attention scores $\sigma_j(w_m)$ is close to zero if $p[j] \neq m$, and are close to $0.5$ otherwise.% In simple terms, the model at position $m$ uses both attention scores at children nodes equally to compute its own value, as it should according to the binary tree structure. Such discrepancy exists because of the constant one value for parities within those parent-children triples: Given $n$ samples, for each parent-children triple $(m, \alpha, \beta)$, we easily obtain $\langle \bm{x}_m, \bm{x}_{\alpha}, \bm{x}_{\beta} \rangle = n$; this signal dominates parities for other groups. However, corruption reduces this value to $(1-q) \times 1 \times n + q \times (-1) \times n = (1-2q)n$. If $q < 0.5$, then clearly this signal is still positive and dominates other interactions and ultimately assigns high attention scores for children positions. Moreover, the value of $\rho$ matters, because as $\rho$ decreases, the dominating signal has a larger magnitude.
    
    For the \textit{only if} direction, we prove the contrapositive: If the conditions do not hold, i.e. $B_m$ is not low enough for an intermediate step $m$. The key part is still analyzing the gradient update, but since the condition fails for $m$, the gradient update for at least one node $j$ such that $p[j] \neq m$ is now equal or higher than updates for children. Consequently, the attention score for that incorrect node is equal or higher than the scores for correct nodes. As a result, at least one CoT step will not learn an accurate predictor. We will show a lemma that non-vanishing error on any CoT node crashes the final predictor and hence prove the failure in this scenario.
\end{proof}

\section{Experiments}

\begin{figure}
    \centering
    \includegraphics[scale=0.25]{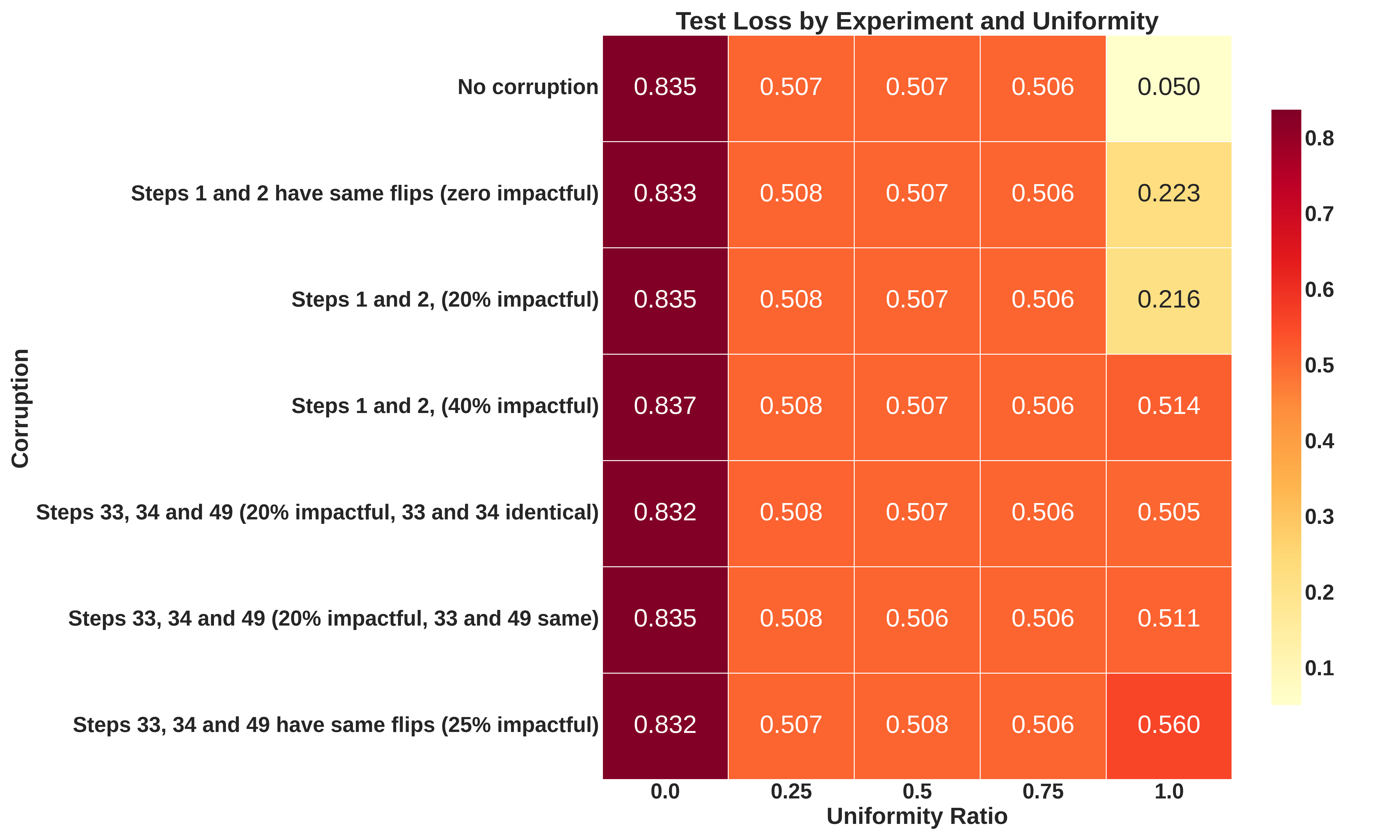}
    \caption{Heatmap of testing loss with respect to both distribution shift and poisoning structure. Each value on the horizontal axis is the ratio of uniformly generated inputs, the same as $\rho$ in Section 4. The number inside each grid is the test loss under that particular circumstance.} 
    \label{fig:heatmap}
\end{figure}

In this section, we provide numerical experiments which support our theoretical analysis. The statement of \Cref{main} holds with a large enough $d$ to overcome low-order error terms and perturbations. In our experiments, we implement a more realistic setting on dimensions and learning rates with an extensive period of training time. We train a simple one-layer transformer with absolute encoding, softmax attention layer and the feedforward layer defined in Section 3.2. The investigated problem is exactly the same as the one in Section 4.2. We set the input dimension to be 128 and by default $k=64$. The value of $k$ can be easily adjusted within 128, but for larger $d$, e.g. 256 or 512, the computation becomes prohibitively slow even with A100 GPU.

Our implementation automatically generates synthetic data from $\{ \pm 1 \}^{128}$ and randomly selects $k$ relevant bits. Once the inputs are generated and relevant bits are selected, the program then constructs a decomposition tree like the illustration in \Cref{fig:cot-tree}. Next, using the inputs, the program computes the ground truth samples for CoT training by multiplying the correct bits element-wise following the decomposition tree structure.

The key parameters for our problem other than the standard ones above are the ratio of uniformly generated inputs and structure/quantity of data poisoning. After the inputs are generated, the program allows us to input the variable \texttt{uniform\char`_prob}, which is $\rho$ defined in Section 4, and then $(1-\texttt{uniform\char`_prob})n$ samples will have their coordinates at target bits to be changed to either all $-1$ or $1$, both with probability 0.5. We may also easily inject poisoning with any quantity and structure by editing the list \texttt{flip\char`_configurations}.

We experimented over 35 cases with seven variants of poisoning quantity and structure and five values of \texttt{uniform\char`_prob}. \Cref{fig:heatmap} shows the training results after 5000 epochs in every case when $d = 128$ and $k = 64$. We first note that the performance under no distribution shift strictly surpasses any other case with distinct poisoning structure and quantity, with the exception of the case where the first and second CoT steps have $40\%$ of ground truth labels flipped. The high loss for this poisoning structure even without distribution shift is justifiable as the level of impact poisoning is high. We also observe that, even without distribution shift, the empirical poisoning tolerance is not as high as $50\%$, and the location of poisoning matters. Theoretically, even with a poisoning of $40\%$, if $d$ is large enough the performance should not be different with the performance under the case without any poisoning, and the location of poisoning should not matter as long as the conditions in \Cref{eq:G(h[m]=1)} and \ref{eq:G(h[m]>1)} matters. But an eligible $d$, as we will show in the proof of \Cref{main}, must be astronomically large and cannot be empirically tested. Nevertheless, the heatmap shows that the performance degrades with low $\rho$, and the impact of poisoning structure is real: The case with $25\%$ of poisoning in steps 33, 34 and 49 has a strictly higher loss than the other two cases with the same poisoning location.

The empirical results strengthens our theoretical discoveries on the relationship between the CoT performance and data shifts. Meticulous assessment of the data shifts is essential to ensure the success of the CoT training with decomposition in Section 3. We have ran the experiments multiple times with stable results, and the details on variability in multiple runs will be presented in the appendix.

\section{Conclusion}

In this work, we provide, to the best of our knowledge, the first theoretical analysis of limitations of CoT applied on a concrete problem. Because the $k$-parity problem is well-defined and based on binary inputs, there are clear measures on success/failure, distribution shift, and poisoning level. We derive a necessary and sufficient condition on the distribution‑shift parameter $\rho$ and structured poisoning levels $q_{a,b,c}$ that rigidly regulates the success of training.

\paragraph{Limitations.} Despite the insights we gained, our study has two major limitations. First, our focus is entirely on the parity problem. In spite of the advantages of the problem, real-world applications of CoT are much more nuanced and may lead to different results. The other constraint is the choice of CoT decomposition. Although this decomposition has been proved to be successful, there might be alternatives which may interact with data shifts differently. Future work on a more diverse set of problems and CoT structures would be inspiring for both theory and applications. 

\bibliographystyle{plainnat}
\bibliography{arxiv_paper_draft}

\newpage

%%%%%%%%%%%%%%%%%%%%%%%%%%%%%%%%%%%%%%%%%%%%%%%%%%%%%%%%%%%%

\appendix

\section*{Technical Appendices and Supplementary Material}
%Technical appendices with additional results, figures, graphs and proofs may be submitted with the paper submission before the full submission deadline (see above), or as a separate PDF in the ZIP file below before the supplementary material deadline. There is no page limit for the technical appendices.

\section{Proof of \Cref{easy-input-output}}

\begin{figure}[h]
    \centering
    \includegraphics[scale=1]{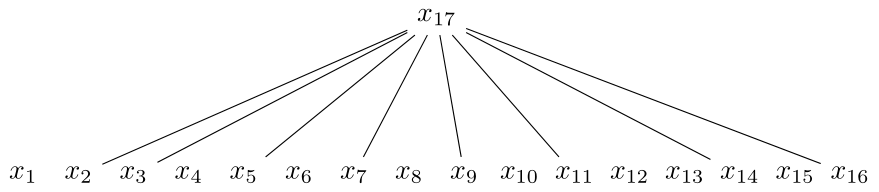}
    \caption{The same parity function where $P = \{ 2,3,5,7,9,11,14,16 \}$ as one represented by \Cref{fig:cot-tree}, but without CoT decomposition.}
    \label{fig:flat-tree}
\end{figure}

By proving \Cref{easy-input-output}, we show that the $k$-parity problem can be solved by simple transformers if the data samples are generated via the distribution $D^P_{\rho}$ for $\rho < 1$.

%\begin{definition}
%    Under distribution $\mathcal{D}$, Samples are uniformly generated with probability $\rho$ instead of 0.5, and with probability $1-\rho$, the relevant bits are either all 1 or -1.
%\end{definition}

First, we visually understand the problem. We illustrated the case with CoT using \Cref{fig:cot-tree}. If we omit CoT and use the simple input-output model, then the tree only has depth one, where the root is the final output and its $k$ children are relevant bits, as illustrated in . Given $n$ samples, the loss function is
\begin{equation}
    L(\mathbf{W}) = \frac{1}{2n} \left \lvert \left \lvert \hat{\bm{x}}_{d+1} - \bm{x}_{d+1} \right \rvert \right \rvert^2_{\infty}.
\end{equation}
Unlike the loss function for the case with CoT, the function here is a single summand because the only prediction during the process is the output. 

Using the softmax attention, some quantities in the original paper preserve. We will need the following expressions. For any $1 \leq \alpha, j < d+1$, denote the $\delta_{j \alpha}$ as the 0-1 indicator on $j = \alpha$, we have
\begin{equation}
    \frac{\partial \sigma_{\alpha}(\bm{w}_{d+1})}{\partial w_{j,d+1}} = (\delta_{j \alpha} - \sigma_{\alpha} (\bm{w}_{d+1})) \sigma_j (\bm{w}_{d+1}) = (\delta_{j \alpha} - \sigma_j (\bm{w}_{d+1})) \sigma_{\alpha} (\bm{w}_{d+1});
\end{equation}
and
\begin{equation}
    \frac{\partial \hat{\bm{z}}_{d+1}}{\partial w_{j,d+1}} = \sum_{\alpha=1}^{d} (\delta_{j \alpha} - \sigma_j (\bm{w}_{d+1})) \sigma_{\alpha} (\bm{w}_{d+1}) \bm{x}_{\alpha} = \sigma_j (\bm{w}_{d+1})(\bm{x}_j - \hat{\bm{z}}_{d+1}).
\end{equation}

\paragraph{Feedforward layer.} Because of the nice properties of the 2-parity problem, we could apply a simple feedforward activation $\phi$ as long as $\phi(0) = -1$ and $\phi(\pm 1) = 1$. However, in this ``flat'' problem, we must choose a feedforward function $\phi$ to satisfy: $\phi \left( \frac{x_1 + \cdots + x_k}{k} \right) = x_1 x_2 \cdots x_k$. This leads to a few other quantitative requirements:
\begin{itemize}[noitemsep, nolistsep, leftmargin=*]
    \item $\phi(0) = 1$. Because if $x_1 + \cdots + x_k = 0$, then we have equally many $1$'s and $-1$'s, so $x_1 \cdots x_k = 1$.
    \item $\phi \left( \pm \frac{2}{k} \right) = \phi \left( \pm \frac{6}{k} \right) = \cdots = \phi \left( \pm \frac{k-2}{k} \right) = -1$. This is the case when the difference between the numbers of $1$'s and $-1$'s is odd, and their product is $-1$.
    \item $\phi \left( \pm \frac{4}{k} \right) = \phi \left( \pm \frac{8}{k} \right) = \cdots = \phi \left( \pm \frac{k}{k} \right) = 1$. This is the case when the difference between the numbers of $1$'s and $-1$'s is even, and their product is $1$.
\end{itemize}
A plausible choice is $\phi(x) = \cos (0.5 k \pi x)$; it satisfies all the properties above. When $x$ is small, using Taylor expansion we can approximate its derivative as $\phi'(x) \approx -0.25 (k^2 \pi^2) x$, and the remaining terms can be approximated as $O(x^3)$. To simplify the notation, we will write $\phi'(x) = -2cx = -2a k^2 x$ where $c = ak^2$ and $a = \pi^2/8$.

Therefore,
\begin{align}
    \frac{\partial L}{\partial w_{j, d+1}} (\mathbf{W}) & = \frac{1}{n} \left( \phi(\hat{\bm{z}}_{d+1}) - \bm{x}_{d+1} \right)^{\top} \frac{\partial \phi(\hat{\bm{z}}_{d+1})}{\partial w_{j, d+1}} = \frac{\sigma_j (\bm{w}_{d+1})}{n} \left \langle \phi(\hat{\bm{z}}_{d+1}) - \bm{x}_{d+1}, \phi' (\hat{\bm{z}}_{d+1}), \bm{x}_j - \hat{\bm{z}}_{d+1} \right \rangle\\
    & = -\frac{1}{nd} \left \langle \bm{x}_{d+1}, -2c \hat{\bm{z}}_{d+1}, \bm{x}_j - \hat{\bm{z}}_{d+1} \right \rangle\\
    & \quad + \frac{1}{nd} \left \langle - \bm{1}_n + c \hat{\bm{z}}_{d+1}^2, 2c \hat{\bm{z}}_{d+1}, \bm{x}_j - \hat{\bm{z}}_{d+1} \right \rangle\\
    & \quad + \frac{1}{nd} \left \langle O(\left \lvert \hat{\bm{z}}_{d+1} \right \rvert^4), -2c \hat{\bm{z}}_{d+1}, \bm{x}_j - \hat{\bm{z}}_{d+1} \right \rangle\\
    & \quad + \frac{1}{nd} \left \langle \phi(\hat{\bm{z}}_{d+1}) - \bm{x}_{d+1}, O(\left \lvert \hat{\bm{z}}_{d+1} \right \rvert^3), \bm{x}_j - \hat{\bm{z}}_{d+1} \right \rangle.
\end{align}

Like the original proof, we will show that the first term is the leading term and the other three vanish eventually. Recall that under the CoT setting, thanks to the binary tree structure, the multi-linear product among a parent and two children is always one.

In our case, however, this no longer holds: Suppose $d=16$, $k=4$, $x_{17}$ is the root, and $x_2, x_3, x_6, x_7$ are relevant bits. Observe that the product $\langle x_{17}, x_2, x_7 \rangle = x_2 x_3 x_6 x_7 x_2 x_7 = x_3 x_6$ is a random variable instead of a fixed value.

Nevertheless, we will see the random variables are largely homogeneous. We first express the leading term in a more readable way by substituting $\hat{\bm{z}}_{d+1} = \frac{1}{d} \sum_{\alpha} \bm{x}_{\alpha}$. Observe that,
\begin{equation} \label{leading-term-sub}
    \frac{1}{n} \left \langle \bm{x}_{d+1}, \hat{\bm{z}}_{d+1}, \bm{x}_j - \hat{\bm{z}}_{d+1} \right \rangle = \frac{1}{nd} \sum_{\alpha} \langle \bm{x}_{d+1}, \bm{x}_{\alpha}, \bm{x}_j \rangle - \frac{1}{nd^2} \sum_{\alpha, \beta} \langle \bm{x}_{d+1}, \bm{x}_{\alpha}, \bm{x}_{\beta} \rangle,
\end{equation}
where the dummy indices $\alpha$ and $\beta$ are taken to run over $[d]$. Before continuing, we prove a lemma that provides a concentration bound for the interactions between the bits.

\begin{lemma} \label{lemma-vanishing-terms}
    Using the distribution $\mathcal{D}$, and let $r \in [4]$, then for each of the following two cases:
    \begin{enumerate}
        \item $r$ is odd
        \item $r$ is even but at least one of $j_1, \ldots, j_r$ is an irrelevant bit.
    \end{enumerate}
    Then for any $p>0$, it holds with probability at least $1-p$ that
    \begin{equation}
        \max_{r \in [4], \{ j_1, \ldots, j_r \nsubseteq P \}} \frac{\lvert \langle \bm{x}_{j_1}, \ldots, \bm{x}_{j_r} \rangle \rvert}{n} \leq \kappa := \sqrt{\frac{2}{n} \log \frac{4d^4}{p}}.
    \end{equation}
    Otherwise, if the indices of $r$ bits are not any one of the two case above, then similarly, for any $p>0$, it holds with probability as least $1-p$ that
    \begin{equation}
        \max_{r \in [4], \{ j_1, \ldots, j_r \nsubseteq P \}} \frac{\lvert \langle \bm{x}_{j_1}, \ldots, \bm{x}_{j_r} \rangle - (1-\rho) \rvert}{n} \leq \kappa := \sqrt{\frac{2}{n} \log \frac{4d^4}{p}}.
    \end{equation}
\end{lemma}

\begin{proof}
    We prove the cases when $r$ is odd or $r$ is even but at least one of $j_1, \ldots, j_r$ is an irrelevant bit. The proof for the remaining scenario with non-zero mean is the same. Observe that, for each sample, the multi-linear product $\langle x_{j_1}, \ldots, x_{j_r} \rangle = x_{j_1} \cdots x_{j_r}$ is a random variable. Suppose $r$ is odd, i.e. $r = 1$ or $3$ and the bits are all relevant. Let $j \in P$, then $\mathbb{P}_{\mathcal{D}} (x_j = 1) = \rho \times 0.5 + (1-\rho) \times 0.5 = 0.5 = \mathbb{P}_{\mathcal{D}} (x_j = -1)$, so $\mathbb{E}_{\mathcal{D}} [x_j] = 0$. Let $j_1, j_2, j_3 \in P$, we have
    \begin{equation}
        \mathbb{P}_{\mathcal{D}} (x_{j_1} x_{j_2} x_{j_3} = 1) = \rho \times \frac{\binom{3}{1} + \binom{3}{3}}{2^3} + (1-\rho) \times \frac{1}{2} = 0.5 = \mathbb{P}_{\mathcal{D}} (x_{j_1} x_{j_2} x_{j_3} = -1).
    \end{equation}
    Therefore, $\mathbb{E}_{\mathcal{D}} [x_{j_1} x_{j_2} x_{j_3}] = 0$.
    
    Next, suppose $j_1 \notin P$, then clearly $\mathbb{P}_{\mathcal{D}} [x_{j_1} = 1] = \mathbb{P}_{\mathcal{D}} [x_{j_1} = -1] = \rho \times 0.5+ (1-\rho) \times 0.5 = 0.5$, so $\mathbb{E}_{\mathcal{D}} [x_{j_1}] = \mathbb{E}_{\text{Uniform}} [x_{j_1}] = 0$. Non-relevant bits are independent with respect to every other bit, so $\mathbb{E}_{\mathcal{D}} \left[ \langle x_{j_1}, \ldots, x_{j_r} \rangle \right] = \mathbb{E}_{\text{Uniform}} \left[ \langle x_{j_1}, \ldots, x_{j_r} \rangle \right] = 0$. 
    
    Since all the variables discussed above have zero mean, by Hoeffding's inequality, we have
    \begin{equation}
        \mathbb{P}_{\mathcal{D}} \left( \lvert \langle x_{j_1}, \ldots, x_{j_r} \rangle \rvert \geq \lambda \right) \leq 2e^{-\lambda^2/2n}.
    \end{equation}
    If $r=1$, there are exactly $d-k$ such variables; if $r=2$, there are $d(k-1)$; if $r=3$, there are $d(k-1)(k-2)$; if $r=4$, there are $d(k-1)(k-2)(k-3)$. Their sum is below $2d^4$. So, by union bounding, we have
    \begin{equation}
        \mathbb{P} \left( \max_{r \in [4], \{ j_1, \ldots, j_r \nsubseteq P \}} \left \lvert \langle \bm{x}_{j_1}, \ldots, \bm{x}_{j_r} \rangle \right \rvert \geq \lambda \right) \leq 4d^4 e^{-\lambda^2/2n}.
    \end{equation}
    The lemma statement follows by substituting $\lambda = \sqrt{\frac{2}{n} \log \frac{4d^4}{p}}$.
\end{proof}

If we take $n = \Omega(d^{2+\epsilon})$ and $p = \exp \left( -d^{\epsilon/2} \right)$, so $\kappa = O \left( d^{-1-\epsilon/4} \right)$.

\subsection{Term (15)}

We now proceed to analyze the quantity of the leading term when $j \in P$, i.e. $j$ is a relevant bit. Thus, we can decompose the two terms in \Cref{leading-term-sub} as:
\begin{align}
    \frac{1}{nd} \sum_{\alpha} \langle \bm{x}_{d+1}, \bm{x}_{\alpha}, \bm{x}_j \rangle & = \frac{1}{nd} \sum_{\alpha \in P} \langle \bm{x}_{d+1}, \bm{x}_{\alpha}, \bm{x}_j \rangle + \frac{1}{nd} \sum_{\alpha \notin P} \langle \bm{x}_{d+1}, \bm{x}_{\alpha}, \bm{x}_j \rangle\\
    & = \frac{1}{nd} \sum_{\alpha \in P} \langle \bm{x}_{d+1}, \bm{x}_{\alpha}, \bm{x}_j \rangle + \frac{1}{d} \cdot O((d-k) \kappa)\\
    & = \frac{1}{d} \cdot X_k + \frac{k-1}{d} \cdot X_{k-2} + \frac{d-k}{d} O(\kappa).
\end{align}
Similarly,
\begin{align}
    \frac{1}{nd^2} \sum_{\alpha} \langle \bm{x}_{d+1}, \bm{x}_{\alpha}, \bm{x}_{\beta} \rangle & = \frac{1}{nd^2} \sum_{\alpha, \beta \in P} \langle \bm{x}_{d+1}, \bm{x}_{\alpha}, \bm{x}_{\beta} \rangle + \frac{1}{nd^2} \sum_{\text{rest}} \langle \bm{x}_{d+1}, \bm{x}_{\alpha'}, \bm{x}_{\beta'} \rangle\\
    & = \frac{1}{nd^2} \sum_{\alpha, \beta \in P} \langle \bm{x}_{d+1}, \bm{x}_{\alpha}, \bm{x}_{\beta} \rangle + O \left( \frac{d^2 - k^2}{d^2} \cdot \kappa \right)\\
    & = \frac{k}{d^2} \cdot X_k + \frac{k(k-1)}{d^2} \cdot X_{k-2} + \frac{d^2-k^2}{d^2} O(\kappa).
\end{align}
Adding them up, we have that, if $j \in P$, then
\begin{equation}
    \frac{1}{n} \left \langle \bm{x}_{d+1}, \hat{\bm{z}}_{d+1}, \bm{x}_j - \hat{\bm{z}}_{d+1} \right \rangle = \frac{d-k}{d^2} X_k + \frac{(d-k)(k-1)}{d^2} X_{k-2} - \frac{k(d-k)}{d^2} O(\kappa).
\end{equation}
Multiplying $\frac{2c}{d} = \frac{2ak^2}{d}$, we finally have
\begin{equation}
    \frac{2c}{nd} \left \langle \bm{x}_{d+1}, \hat{\bm{z}}_{d+1}, \bm{x}_j - \hat{\bm{z}}_{d+1} \right \rangle = \frac{2ak^2(d-k)}{d^3} X_k + \frac{2ak^2(d-k)(k-1)}{d^3} X_{k-2} - \frac{2ak^3(d-k)}{d^3} O(\kappa).
\end{equation}

On the other hand, if $j \notin P$, then we have
\begin{align}
    \frac{1}{n} \langle \bm{x}_{d+1}, \hat{\bm{z}}_m, \bm{x}_j \rangle & = \frac{1}{nd} \sum_{\alpha} \langle \bm{x}_{d+1}, \bm{x}_{\alpha}, \bm{x}_j \rangle - \frac{1}{nd^2} \sum_{\alpha, \beta} \langle \bm{x}_{d+1}, \bm{x}_{\alpha}, \bm{x}_{\beta} \rangle\\
    & = \frac{1}{d} X_k + \frac{d-1}{d} O(\kappa) - \left( \frac{k}{d^2} X_k + \frac{k(k-1)}{d^2} X_{k-2} + \frac{d^2-k^2}{d^2} O(\kappa) \right)\\
    & = \frac{d-k}{d^2} X_k - \frac{k(k-1)}{d^2} X_{k-2} + \frac{k^2-d}{d^2} O(\kappa).
\end{align}
Again, multiplying $\frac{2c}{d} = \frac{2ak^2}{d}$, we have
\begin{equation}
    \frac{2c}{nd} \langle \bm{x}_{d+1}, \hat{\bm{z}}_m, \bm{x}_j \rangle = \frac{2ak^2(d-k)}{d^3} X_k - \frac{2ak^3(k-1)}{d^3} X_{k-2} + \frac{2ak^2(k^2-d)}{d^3} O(\kappa).
\end{equation}

Therefore, we have
\begin{align}
     \text{Term (15)} =
     \begin{cases}
        \frac{2ak^2(d-k)}{d^3} X_k + \frac{2ak^2(d-k)(k-1)}{d^3} X_{k-2} - \frac{2ak^3(d-k)}{d^3} O(\kappa), \quad j \in P\\
        \frac{2ak^2(d-k)}{d^3} X_k - \frac{2ak^3(k-1)}{d^3} X_{k-2} + \frac{2ak^2(k^2-d)}{d^3} O(\kappa), \quad j \notin P
     \end{cases}.
\end{align}

%We will precisely control the quantities of the terms $\sum_{\alpha \in P} \langle \bm{x}_{d+1}, \bm{x}_{\alpha}, \bm{x}_j \rangle$ and $\sum_{\alpha, \beta \in P} \langle \bm{x}_{d+1}, \bm{x}_{\alpha}, \bm{x}_{\beta} \rangle$ later. Before that, we first analyze the other three terms in $\partial L$.

\subsection{Term (16)}

We expand term (16) as the following:
\begin{equation}
    \frac{1}{nd} \left \langle - \bm{1}_n + c \hat{\bm{z}}_{d+1}^2, 2c \hat{\bm{z}}_{d+1}, \bm{x}_j - \hat{\bm{z}}_{d+1} \right \rangle = -\frac{2c}{nd} \left \langle \hat{\bm{z}}_{d+1}, \bm{x}_j \right \rangle + \frac{2c}{nd} \left \langle \hat{\bm{z}}_{d+1}^2 \right \rangle + \frac{2c^2}{nd} \left \langle \hat{\bm{z}}_{d+1}^3, \bm{x}_j \right \rangle - \frac{2c^2}{nd} \left \langle \hat{\bm{z}}_{d+1}^4 \right \rangle.
\end{equation}

\subsubsection{First term}

\begin{align}
    \frac{1}{n} \left \langle \hat{\bm{z}}_{d+1}, \bm{x}_j \right \rangle & = \frac{1}{nd} \left( \langle \bm{x}_j, \bm{x}_j \rangle + \sum_{\alpha \neq j} \langle \bm{x}_{\alpha}, \bm{x}_j \rangle \right) = \frac{1}{d} + \frac{1}{nd} \sum_{\alpha \neq j} \langle \bm{x}_{\alpha}, \bm{x}_j \rangle\\
    & =
    \begin{cases}
        1/d + \frac{1}{nd} \sum_{\alpha \in P \setminus \{j \}} \langle \bm{x}_{\alpha}, \bm{x}_j \rangle, \quad j \in P\\
        1/d + \frac{d-1}{d} O(\kappa), \quad j \notin P
    \end{cases}\\
    & =
    \begin{cases}
        1/d + \frac{k-1}{d} X_2 + \frac{d-k}{d} O(\kappa), \quad j \in P\\
        1/d + \frac{d-1}{d} O(\kappa), \quad j \notin P
    \end{cases}
    %& = 
    %\begin{cases}
    %    1/d + O(1) \cdot X_2 + O(\kappa), \quad j \in P\\
    %    1/d + O(\kappa), \quad j \notin P
    %\end{cases}
\end{align}

Multiplying the constant $-2c$, we have
\begin{align}
    -\frac{2c}{nd} \left \langle \hat{\bm{z}}_{d+1}, \bm{x}_j \right \rangle & =
    \begin{cases}
        -\frac{2c}{d^2} - \frac{2c(k-1)}{d^2} X_2 - \frac{2c(d-k)}{d^2} O(\kappa), \quad j \in P\\
        -\frac{2c}{d^2} - \frac{2c(d-1)}{d^2} O(\kappa), \quad j \notin P
    \end{cases}\\
    & = \begin{cases}
        -\frac{2ak^2}{d^2} - \frac{2ak^2(k-1)}{d^2} X_2 - \frac{2ak^2(d-k)}{d^2} O(\kappa), \quad j \in P\\
        -\frac{2ak^2}{d^2} - \frac{2ak^2(d-1)}{d^2} O(\kappa), \quad j \notin P
    \end{cases}
\end{align}

\subsubsection{Second term}
Next,
\begin{align}
    \frac{1}{n} \left \langle \hat{\bm{z}}_{d+1}^2 \right \rangle & = \frac{1}{nd^2} \left( \sum_{\alpha} \langle \bm{x}_{\alpha}, \bm{x}_{\alpha} \rangle + \sum_{\alpha \neq \beta} \langle \bm{x}_{\alpha}, \bm{x}_{\beta} \rangle \right)\\
    & = \frac{1}{d} + \frac{1}{nd^2} \sum_{\alpha \neq \beta; \alpha, \beta \in P} \langle \bm{x}_{\alpha}, \bm{x}_{\beta} \rangle + \frac{1}{nd^2} \sum_{\text{rest}} \langle \bm{x}_{\alpha}, \bm{x}_{\beta} \rangle\\
    & = \frac{1}{d} + \frac{k(k-1)}{d^2} X_2 + \frac{(d+k-1)(d-k)}{d^2} O(\kappa)
    %& = \frac{1}{d} + O(1) \cdot X_2 + O(\kappa).
\end{align}
Therefore,
\begin{align}
    \frac{2c}{nd} \left \langle \hat{\bm{z}}_{d+1}^2 \right \rangle & = \frac{2c}{d^2} + \frac{2ck(k-1)}{d^3} X_2 + \frac{2c(d+k-1)(d-k)}{d^3} O(\kappa)\\
    & = \frac{2ak^2}{d^2} + \frac{2ak^3(k-1)}{d^3} X_2 + \frac{2ak^2(d+k-1)(d-k)}{d^3} O(\kappa).
\end{align}

\subsubsection{Fourth term}

For the fourth order term, we have
\begin{equation}
    \frac{1}{n} \left \langle \hat{\bm{z}}_{d+1}^4 \right \rangle = \frac{1}{nd^4} \sum_{\alpha, \beta, \gamma, \delta} \langle \bm{x}_{\alpha}, \bm{x}_{\beta}, \bm{x}_{\gamma}, \bm{x}_{\delta} \rangle
\end{equation}

We analyze all possible combinations of the indices by enumerating the size of the set $\{ \alpha, \beta, \gamma, \delta \}$.
\begin{enumerate}[noitemsep, nolistsep, leftmargin=*]
    \item $\lvert \{ \alpha, \beta, \gamma, \delta \} \rvert = 1$: There are $d$ instances of $\langle \bm{x}_{\alpha}, \bm{x}_{\alpha}, \bm{x}_{\alpha}, \bm{x}_{\alpha} \rangle$.
    \item $\lvert \{ \alpha, \beta, \gamma, \delta \} \rvert = 2$: There are $\binom{d}{2}$ pairs of distinct indices. For each pair $(\alpha, \beta)$, either:
    \begin{enumerate}[noitemsep, nolistsep, leftmargin=*]
        \item One appears three times. There are eight possibilities: We first choose the position of the unique index (four), and both can be that unique index (multiplying two).
        \item Both appear twice. Clearly, there are six possibilities.
    \end{enumerate}
    Therefore, there are totally $14 \binom{d}{2} = 7d(d-1)$ possible combinations in this case.
    \item $\lvert \{ \alpha, \beta, \gamma, \delta \} \rvert = 3$: Given a triple $(\alpha, \beta, \gamma)$, exactly one index must appear twice, so there are three possibilities for this criteria. Once this choice is made, the identical indices may choose one of the six pairs among the four positions. Once two spots are occupied, the remaining two indices may fill in either order. So, there are totally $36 \binom{d}{3} = 6d(d-1)(d-2)$ possible combinations in this case.
    \item $\lvert \{ \alpha, \beta, \gamma, \delta \} \rvert = 4$: Clearly, there are $\binom{d}{4}$ groups of all distinct indices. For each group, they may be placed in $4! = 24$ possible orders. So there are totally $24 \binom{d}{4} = d(d-1)(d-2)(d-3)$ possible combinations in this case.
\end{enumerate}

Every instance in Cases 1 and 2.b has value exactly $n$, and there are $d + 6 \binom{d}{2} = 3d^2-2d$ such instances. Every instances in Cases 2.a and 3 is a sum of $n$ independent (data samples) random variables, and each of them is in the form $\langle x_{\alpha}, x_{\beta} \rangle$ such that $\alpha \neq \beta$. Every instance in Case 4 is a sum of $n$ independent random variables in the form $\langle x_{\alpha}, x_{\beta}, x_{\gamma}, x_{\delta} \rangle$ such that $\lvert \{ \alpha, \beta, \gamma, \delta \} \rvert = 4$. So we further divide those $d^4$ instance into three groups:
\begin{itemize}[noitemsep, nolistsep, leftmargin=*]
    \item Cases 1 and 2.b.
    \item Cases 2.a and 3. Among $6d^3-14d^2+8d$ such instances, exactly $8 \binom{k}{2} + 36 \binom{k}{3} = 6k^3 - 14k^2 + 8k$ correspond to random variables as multiplication of two bits in $P$. The rest of them have value $O(\kappa)$.
    \item Case 4. Clearly $24 \binom{k}{4} = k(k-1)(k-2)(k-3)$ instances correspond to random variables as multiplication of four bits in $P$. The rest of them have value $O(\kappa)$.
\end{itemize}
The first group with $3d^2-2d$ instances accumulates to
\begin{equation}
    \frac{1}{nd^4} \times (\text{Cases 1 and 2.b}) = \frac{1}{nd^4} \times n(3d^2-2d) = \frac{3d-2}{d^3}.
\end{equation}
The second group leads to
\begin{align}
    \frac{1}{nd^4} \times (\text{Cases 2.a and 3}) = \frac{6k^3-14k^2+8k}{d^4} \cdot X_2 + \frac{6(d^3-k^3)-14(d^2-k^2)+8(d-k)}{d^4} O(\kappa).
    %& = \frac{6k^3-14k^2+8k}{d^4} \cdot X_2 + O( \kappa/d).
\end{align}
The third group leads to
\begin{align}
    \frac{1}{nd^4} \times (\text{Case 4}) = \frac{k(k-1)(k-2)(k-3)}{d^4} \cdot X_4 + \frac{d(d-1)(d-2)(d-3)-k(k-1)(k-2)(k-3)}{d^4} O(\kappa).
    %& = O(1) \cdot X_4 + O(\kappa).
\end{align}
Adding everything up, we have
\begin{align}
    \frac{1}{n} \left \langle \hat{\bm{z}}_{d+1}^4 \right \rangle & = \frac{3d-2}{d^3} + \frac{6k^3-14k^2+8k}{d^4} \cdot X_2 + \frac{k(k-1)(k-2)(k-3)}{d^4} \cdot X_4\\
    & \quad + \frac{(d^4-k^4)-3(d^2-k^2)+2(d-k)}{d^4} O(\kappa).
    %& \quad + \frac{d(d-1)(d-2)(d-3)-k(k-1)(k-2)(k-3) + 6(d^3-k^3)-14(d^2-k^2)+8(d-k)}{d^4} O(\kappa).
\end{align}
Multiplying $-2c^2 = -2 a^2 k^4$, we have
\begin{align}
    -\frac{2c^2}{nd} \left \langle \hat{\bm{z}}_{d+1}^4 \right \rangle & = -\frac{2a^2 k^4(3d^2-2d)}{d^5} - \frac{2a^2 k^4(6k^3-14k^2+8k)}{d^5} \cdot X_2 - \frac{2a^2 k^5(k-1)(k-2)(k-3)}{d^5} \cdot X_4\\
    & \quad - \frac{2a^2 k^4 \left[(d^4-k^4)-3(d^2-k^2)+2(d-k) \right]}{d^5} \cdot O(\kappa).
\end{align}

\subsubsection{Third term}
Finally, for the three order term, where the group $\{ \alpha, \beta, \gamma, \delta \}$ must contain a fixed $j \in [d]$. We have,
\begin{equation}
    \frac{1}{n} \left \langle \hat{\bm{z}}_{d+1}^3, \bm{x}_j \right \rangle = \frac{1}{nd^3} \sum_{\alpha, \beta, \gamma} \langle \bm{x}_{\alpha}, \bm{x}_{\beta}, \bm{x}_{\gamma}, \bm{x}_j \rangle
\end{equation}
Like the previous four order term, we need to discuss multiple cases.
\begin{enumerate}[noitemsep, nolistsep, leftmargin=*]
    \item $\alpha = \beta = \gamma = j$. We easily obtain $\frac{1}{nd^3} \langle \bm{x}_j, \bm{x}_j, \bm{x}_j, \bm{x}_j \rangle = 1/d^3$.
    \item $\alpha = \beta = \gamma \neq j$. There are $d-1$ possible combinations.
    \item $\lvert \{ \alpha, \beta, \gamma \} \rvert = \lvert \{ \alpha, \beta, \gamma, j \} \rvert = 2$. This further divides into two sub-cases.
    \begin{enumerate}[noitemsep, nolistsep, leftmargin=*]
        \item Only one of them is $j$. We have $d-1$ choices for $\alpha \neq j$, and three choices of spot for that unique $j$. Totally, there are $3(d-1)$ combinations.
        \item Two of them are $j$. Like the sub-case above, there are $3(d-1)$ possible combinations.
    \end{enumerate}
    \item $\lvert \{ \alpha, \beta, \gamma \} \rvert = 2$, but $\lvert \{ \alpha, \beta, \gamma, j \} \rvert = 3$. None of the three indices can be $j$, so there are $\binom{d-1}{2}$ pairs of $(\alpha, \beta)$. For each pair, we need to choose which index appears once between two elements. In total, there are $6 \binom{d-1}{2} = 3(d-1)(d-2)$ combinations.
    \item $\lvert \{ \alpha, \beta, \gamma \} \rvert = \lvert \{ \alpha, \beta, \gamma, j \} \rvert = 3$. Exactly one of the three indices must be $j$, so there are $\binom{d-1}{2}$ pairs of unequal indices. There are three choices for $j$'s position, and every time we can flip the pair. Therefore, there are $6 \binom{d-1}{2} = 3(d-1)(d-2)$ combinations in this case.
    \item $\lvert \{ \alpha, \beta, \gamma \} \rvert = 3$, and $\lvert \{ \alpha, \beta, \gamma, j \} \rvert = 4$. There are $\binom{d-1}{3}$ triples $(\alpha, \beta, \gamma)$ and six combinations for each choice, so there are $6 \binom{d-1}{3} = (d-1)(d-2)(d-3)$ combinations in total.
\end{enumerate}

The unique case in Case 1 is trivial.

We first assume $j \notin P$. Then \Cref{lemma-vanishing-terms} directly applies to all summands and therefore the term sums to the following with high probability:
\begin{equation}
    \frac{2c^2}{nd} \left \langle \hat{\bm{z}}_{d+1}^3, \bm{x}_j \right \rangle = \frac{2c^2}{nd^4} \sum_{\alpha, \beta, \gamma} \langle \bm{x}_{\alpha}, \bm{x}_{\beta}, \bm{x}_{\gamma}, \bm{x}_j \rangle = \frac{2c^2 \cdot \langle \bm{x}_{\alpha}, \bm{x}_{\beta}, \bm{x}_{\gamma}, \bm{x}_j \rangle}{nd^4} \times d^3 = \frac{2c^2}{d} \cdot O(\kappa) = \frac{2a^2 k^4}{d} O(\kappa).
\end{equation}

Now suppose $j \in P$. Exactly $k-1$ combinations in Case 2 correspond to a sum of $n$ independent $X_2$, and the remaining $d-k$ combinations are $O(\kappa)$. All instances in Case 3.a have product value $n$, and all instances in Case 3.b correspond to $n$ copies of $X_2$. For Case 4, only $j$ and the unique index matter. Exactly $2(k-1)(d-2)$ combinations correspond to $n$ copies of $X_2$, and the rest are $O(\kappa)$. For Case 5, the $x_j$ cancels out, so the sum for Case 5 can be expressed as
\begin{equation}
    \sum_{\alpha \neq \beta; \alpha, \beta \in [d] \setminus \{j \}} \langle \bm{x}_{\alpha}, \bm{x}_{\beta} \rangle.
\end{equation}
Clearly, only $3(k-1)(k-2)$ summands correspond to $n$ copies of $X_2$, and others are $O(\kappa)$. Finally, for Case 6, there are $6 \binom{k-1}{3}$ combinations that correspond to $n$ copies of $X_4$, and all others are $O(\kappa)$. We can then express the three order term as the following.
\begin{align}
    \frac{1}{n} \left \langle \hat{\bm{z}}_{d+1}^3, \bm{x}_j \right \rangle & = \frac{1}{nd^3} \sum_{\alpha, \beta, \gamma} \langle \bm{x}_{\alpha}, \bm{x}_{\beta}, \bm{x}_{\gamma}, \bm{x}_j \rangle\\
    & = \frac{3d-2}{d^3} + \frac{(k-1)+3(d-1)+2(k-1)(d-2)+3(k-1)(k-2)}{d^3} \cdot X_2\\
    & \quad + \frac{(k-1)(k-2)(k-3)}{d^3} \cdot X_4\\
    & \quad + \frac{(d^3-k^3)-2dk+3k^2-4d+k+2}{d^3} \cdot O(\kappa).
\end{align}
%We can further simplify the expression as
%\begin{equation}
%    \frac{1}{n} \left \langle \hat{\bm{z}}_{d+1}^3, \bm{x}_j \right \rangle = \frac{1}{nd^3} \sum_{\alpha, \beta, \gamma} \langle \bm{x}_{\alpha}, \bm{x}_{\beta}, \bm{x}_{\gamma}, \bm{x}_j \rangle = O(1/d^2) + O(1/d) \cdot X_2 + O(1) \cdot X_4 + O(\kappa).
%\end{equation}    
%\begin{equation}
%    \frac{1}{n} \left \langle \hat{\bm{z}}_{d+1}^3, \bm{x}_j \right \rangle = 
%    \begin{cases}
%        O(1/d^2) + O(1/d) \cdot X_2 + O(1) \cdot X_4 + O(\kappa), \quad j \in P\\
%        O(\kappa), \quad j \notin P
%    \end{cases}.
%\end{equation}
Multiplying $2c^2 = 2 a^2 k^4$, we have
\begin{align}
    \frac{2c^2}{nd} \left \langle \hat{\bm{z}}_{d+1}^3, \bm{x}_j \right \rangle & = \frac{1}{nd^4} \sum_{\alpha, \beta, \gamma} \langle \bm{x}_{\alpha}, \bm{x}_{\beta}, \bm{x}_{\gamma}, \bm{x}_j \rangle\\
    & = \frac{2 a^2 k^4 (3d-2)}{d^4} + \frac{2 a^2 k^4 \left[ (k-1)+3(d-1)+2(k-1)(d-2)+3(k-1)(k-2) \right]}{d^4} \cdot X_2\\
    & \quad + \frac{2 a^2 k^4 (k-1)(k-2)(k-3)}{d^4} \cdot X_4\\
    & \quad + \frac{2 a^2 k^4 \left[ (d^3-k^3)-2dk+3k^2-4d+k+2 \right]}{d^4} \cdot O(\kappa).
\end{align}

\subsubsection{Adding everything up}

Adding everything up, we have
\begin{align}
    \text{Term (16)} = -\frac{2c}{nd} \left \langle \hat{\bm{z}}_{d+1}, \bm{x}_j \right \rangle + \frac{2c}{nd} \left \langle \hat{\bm{z}}_{d+1}^2 \right \rangle + \frac{2c^2}{nd} \left \langle \hat{\bm{z}}_{d+1}^3, \bm{x}_j \right \rangle - \frac{2c^2}{nd} \left \langle \hat{\bm{z}}_{d+1}^4 \right \rangle.
\end{align}
Specifically,
\begin{itemize}[noitemsep, nolistsep, leftmargin=*]
    \item If $j \in P$, then
    \begin{itemize}[noitemsep, nolistsep, leftmargin=*]
        \item The coefficient for constant is 0.
        \item The coefficient for $X_2$ is
        \begin{align}
            %\frac{(4a^2 \cdot d^2 k^5 + 6a^2 \cdot d k^6 - 12 a^2 \cdot k^7) + (-2a \cdot d^3 k^3 + 2a^2 \cdot d^2 k^4 - 24a^2 \cdot d k^5 + 28a^2 \cdot k^6) + (2a \cdot d^3 k^2 - 2a \cdot d^2 k^3 + 12a^2 \cdot d k^4 - 16a^2 k^5)}{d^5} = O(d^2).
            & \frac{(4a^2 \cdot d^2 k^5 + 6a^2 \cdot d k^6 - 12 a^2 \cdot k^7) + (-2a \cdot d^3 k^3 + 2a^2 \cdot d^2 k^4 - 24a^2 \cdot d k^5 + 28a^2 \cdot k^6)}{d^5}\\
            & + \frac{(2a \cdot d^3 k^2 - 2a \cdot d^2 k^3 + 12a^2 \cdot d k^4 - 16a^2 k^5)}{d^5}= O(d^2).
        \end{align}
        \item The coefficient for $X_4$ is
        \begin{equation}
            \frac{2a k^4 (d-k)(k-1)(k-2)(k-3)}{d^5} \geq 0; \quad \Rightarrow O(d^3).
        \end{equation}
        \item The coefficient for $O(\kappa)$ is
        \begin{align}
            & \frac{-4 a^2 d^2 k^5 - 2 a^2 d^2 k^4 - 2 a^2 d k^7 + 6 a^2 d k^6 + 2 a^2 d k^5 + 2 a^2 k^8 - 6 a^2 k^6 + 4 a^2 k^5}{d^5}\\
            & + \frac{2 a d^3 k^3 - 2 a d^3 k^2 - 2 a d^2 k^3 + 2 a d^2 k^2}{d^5} = -O(d^2).
        \end{align}
        Hence, this term becomes $-O(d^{1-\epsilon/4})$.
    \end{itemize}
    \item If $j \notin P$, then
    \begin{itemize}[noitemsep, nolistsep, leftmargin=*]
        \item The coefficient for constant is
        \begin{equation}
            -\frac{2a^2 k^4 (3d-2)}{d^4} = -O(d).
        \end{equation}
        \item The coefficient for $X_2$ is
        \begin{equation}
            \frac{-12 a^2 \cdot k^7 + 2a \cdot d^2 k^4 + 28a^2 \cdot k^6 - 2a \cdot d^2 k^3 - 16 a^2 \cdot k^5}{d^5} = -O(d^2).
        \end{equation}
        \item The coefficient for $X_4$ is
        \begin{equation}
            -\frac{2a^2 k^5(k-1)(k-2)(k-3)}{d^5} = -O(d^3).
        \end{equation}
        \item The coefficient for $O(\kappa)$ is
        \begin{equation}
            \frac{6 a^2 d^2 k^4 - 4 a^2 d k^4 + 2 a^2 k^7 - 6 a^2 k^6 + 4 a^2 k^5 - 2 a d^2 k^4 + 2 a d^2 k^3}{d^5} = O(d^2).
        \end{equation}
        Hence, this term becomes $O(d^{1-\epsilon/4})$.
    \end{itemize}
\end{itemize}

\subsection{Terms (17) and (18)}

Now we analyze term (17). Recall that $\left \langle \left \lvert \hat{\bm{z}}_{d+1} \right \rvert^4 \right \rangle = \left \langle \hat{\bm{z}}_{d+1}^4 \right \rangle = O(n/d^2) + O(n/d) \cdot X_2 + O(n) \cdot X_4 + O(n \kappa)$. Observe that each component of $\hat{\bm{z}}_{d+1}$ and $\bm{x}_j - \hat{\bm{z}}_{d+1}$ are contained in $[-1, 1]$ and $[-2, 2]$ respectively, we have
\begin{align}
    & \frac{1}{nd} \left \langle O(\left \lvert \hat{\bm{z}}_{d+1} \right \rvert^4), -2c \hat{\bm{z}}_{d+1}, \bm{x}_j - \hat{\bm{z}}_{d+1} \right \rangle = -\frac{4c}{nd} \cdot O \left( \left \langle \left \lvert \hat{\bm{z}}_{d+1} \right \rvert^4 \right \rangle \right)\\
    & = -\frac{4ak^2}{d} \left( \frac{3d-2}{d^3} + \frac{6k^3-14k^2+8k}{d^4} \cdot X_2 + \frac{k(k-1)(k-2)(k-3)}{d^4} \cdot X_4 + O(d^{-2-\epsilon/4}) \right)\\
    & = -O(d^{-1}) - O(1) \cdot X_2 - O(d) \cdot X_4 - O(\kappa).
\end{align}
Similarly, using the Cauchy-Schwarz inequality, we may bound the final term (18).
\begin{align}
    \frac{1}{nd} \left \langle \phi(\hat{\bm{z}}_{d+1}) - \bm{x}_{d+1}, O(\left \lvert \hat{\bm{z}}_{d+1} \right \rvert^3), \bm{x}_j - \hat{\bm{z}}_{d+1} \right \rangle & = \frac{4}{nd} O \left( \left \langle \left \lvert \hat{\bm{z}}_{d+1} \right \rvert \right \rangle^3 \right)\\
    & \leq \frac{4}{nd} \left \langle \hat{\bm{z}}_{d+1}^2 \right \rangle^{1/2} \left \langle \hat{\bm{z}}_{d+1}^4 \right \rangle^{1/2}\\
    & \leq O \left( d^{1+0.5\epsilon} \right)
\end{align}

Now, summing up the terms, we conclude that for relevant bits $j \in P$, its gradient $\partial L / \partial w_{d+1, j}$ has the dominating term $O(d^3) \cdot X_4 = O(d^3)$; for non-relevant bits $j' \notin P$, its gradient $\partial L / \partial w_{d+1, j'}$ has the dominating term $O(d^{1-\epsilon/4})$. Fix a learning rate $\eta = \Theta(d^{-3+\epsilon/8})$, then we obtain the following comparisons of the weights $\bf{W}^{(1)}$ after one gradient update. For $j \in P$ and $j' \notin P$, we have
\begin{equation}
    \frac{\sigma_{j'}(w^{(1)}_{d+1})}{\sigma_{j}(w^{(1)}_{d+1})} = e^{w^{(1)}_{d+1, j'} - w^{(1)}_{d+1, j}} \leq \exp \left( - \Omega(d^{\epsilon/8}) \right).
\end{equation}
Since attention scores sum up to one, we have $\sum_{j \in P} \sigma_j (\bm{w}^{(1)})$
If both $j, k \in P$, then the higher order terms cancel out and the perturbation terms for correct gradient updates become $O(d^{-\epsilon/4})$. Therefore, we have
\begin{equation}
    \frac{\sigma_{j}(w^{(1)}_{d+1})}{\sigma_{k}(w^{(1)}_{d+1})} = \frac{\sigma_{k}(w^{(1)}_{d+1})}{\sigma_{j}(w^{(1)}_{d+1})} \leq \exp (O(d^{-2 - \epsilon/8})) \leq 1 + O(d^{-2 - \epsilon/8}),
\end{equation}
where the last inequality holds because $e^t \leq 1 + O(t)$ for small $t > 0$. The ratio holds for all $k$ elements in $P$, so for any $j \in P$, we have
\begin{equation}
    \frac{1}{k} - O(d^{-2 - \epsilon/8}) \leq \sigma_{j} (\bm{w}^{(1)}) \leq \frac{1}{k} + O(d^{-2 - \epsilon/8}).
\end{equation}
Therefore, for any $d$-dimensional input $\bm{x}$, the prediction $\hat{y} = \hat{x}_{d+1}$ satisfies the following inequality:
\begin{align}
    | y - \hat{y} | & = | \phi(\hat{z}_{d+1}) - \phi(z_{d+1}) |\\
    & \leq k \times | \hat{z}_{d+1} - z_{d+1} |\\
    & = k \times \left \lvert O(d^{-2 - \epsilon/8}) + (d-k) \times \exp \left( - \Omega(d^{\epsilon/8}) \right) \right \rvert\\
    & = k \cdot O(d^{-2 - \epsilon/8}) = O(d^{-1-\epsilon/8}).
\end{align}

%\lang{After adding up, the signals for relevant bits are positive, but those for irrelevant bits are negative. This is already amplified by the choice of $\phi$ (derivative depending on $k$); a proper choice of learning rate will certainly amplify it further. Todo:
%\begin{itemize}
%    \item Find a specific learning rate so that each attention score $\sigma$ for a relevant bit satisfies $\lvert \sigma - 1/k \rvert = O(1/k^{2+\epsilon})$.
%\end{itemize}
%The theory part is done after this part. We can use softmin: with a positive learning rate (possibly a large one), the weights for relevant bits are very negative and others are very positive. Softmin extracts information from the smallest weights.
%}

\section{Proof of \Cref{main}}

The high-level ideas are identical with the previous proof. We still apply \Cref{lemma-vanishing-terms} for this proof to bound the perturbation terms. To reach the conclusion, we must compute the derivative of loss with respect to weights:
\begin{align}
    \frac{\partial L}{\partial w_{j,m}} (\mathbf{W}) & = \frac{1}{n(m-1)} \left \langle \phi(\hat{\bm{z}}_m - \bm{x}_m), \phi'(\hat{\bm{z}}_m), \bm{x}_j - \hat{\bm{z}}_m \right \rangle\\
    & = \frac{1}{n(m-1)} \left \langle 2 \hat{\bm{z}}_m - 1 - \bm{x}_m, \phi'(\hat{\bm{z}}_m), \bm{x}_j - \hat{\bm{z}}_m \right \rangle.
\end{align}
We can write $\phi'(\hat{z}_m) = 2d^3 \hat{z}_m$, so the derivative becomes
\begin{align}
    \frac{\partial L}{\partial w_{j,m}} (\mathbf{W}) & = \frac{1}{n(m-1)} \left \langle 2 \hat{\bm{z}}_m - 1 - \bm{x}_m, 2d^3 \hat{\bm{z}}_m, \bm{x}_j - \hat{\bm{z}}_m \right \rangle\\
    & = -\frac{1}{n(m-1)} \left \langle \bm{x}_m, 2d^3 \hat{\bm{z}}_m, \bm{x}_j - \hat{\bm{z}}_m \right \rangle\\
    & \quad \quad + \frac{1}{n(m-1)} \left \langle -\bm{1}_n + 2\hat{\bm{z}}_m, 2d^3 \hat{\bm{z}}_m, \bm{x}_j - \hat{\bm{z}}_m \right \rangle
\end{align}
The quadratic derivative only holds in $[-d^{-3}, d^{-3}]$, and we will bound it asymptotically, we may replace it with 2 here.

The structure of the proof is then divided into three parts.
\begin{itemize}[noitemsep, nolistsep, leftmargin=*]
    \item Sections B.1-B.3 are computations of the gradients and gradient differences that lead to conditions in \Cref{condition-one-height} and (\ref{condition-higher-height}).
    \item Section B.4 is the \textit{if} direction of \Cref{main}.
    \item Section B.5 is the \textit{only if} direction of the theorem.
\end{itemize}

\subsection{The first term}

We first rewrite the first term as
\begin{equation}
    -\frac{2}{n(m-1)} \left \langle \bm{x}_m, \hat{\bm{z}}_m, \bm{x}_j - \hat{\bm{z}}_m \right \rangle = - \frac{2}{n(m-1)^2} \sum_{\alpha} \langle \bm{x}_m, \bm{x}_{\alpha}, \bm{x}_j \rangle + \frac{2}{n(m-1)^2} \sum_{\alpha, \beta} \langle \bm{x}_m, \bm{x}_{\alpha}, \bm{x}_{\beta} \rangle.
\end{equation}
The value of the single-sum term depends on the position of $m$ and $j$. We compute the value of $\frac{\sum_{\alpha} \langle \bm{x}_m, \bm{x}_{\alpha}, \bm{x}_j \rangle}{n(m-1)}$ over all six possible cases.
\begin{enumerate}[noitemsep, nolistsep, leftmargin=*]
    \item $h[m]=1$. This condition restricts $d < m < d + k/2$. This large case can be divided into three following sub-cases.
    \begin{enumerate}[noitemsep, nolistsep, leftmargin=*]
        \item Node $j$ is a child of $m$, i.e. $p[j]=m$. In this case, if $\alpha = j'$ is another child of $m$, then $\langle x_m, x_{j'}, x_j \rangle = 1 - 2q_{m, j', j}$. For all other $\alpha \in P \setminus \{ j \}$, observe that the inner product $\langle x_m, x_{\alpha}, x_j \rangle$ is the product of an even number of relevant input variables, so it is a random variable with mean $1-\rho$. On the other hand, if $d < \alpha < m$, then $\langle x_m, x_{\alpha}, x_j \rangle$ is the product of an odd number of relevant input variables, so it is a random variable with mean zero. Therefore, we have
        \begin{equation} \label{Term-1-h[m]=1-correct}
            \frac{1}{n(m-1)} \sum_{\alpha} \langle \bm{x}_m, \bm{x}_{\alpha}, \bm{x}_j \rangle = \frac{1-2q_{m, c_1[m], c_2[m]}}{m-1} + \sum_{\alpha \in P \setminus \{ j \}} \frac{1-2q_{m,\alpha,j}}{m-1} (1-\rho) + O(\kappa).
        \end{equation}
        \item If $h[j]=0$ but $p[j] \neq m$. In this case, the inner product $\langle x_m, x_{\alpha}, x_j \rangle$ is never a determined value. Instead, $\langle x_m, x_{\alpha}, x_j \rangle$ is the product of an even number of relevant input variables whenever $\alpha \in P$, and a variable of mean zero otherwise. Therefore,
        \begin{equation} \label{Term-1-h[m]=1-incorrect-lower-height}
            \frac{1}{n(m-1)} \sum_{\alpha} \langle \bm{x}_m, \bm{x}_{\alpha}, \bm{x}_j \rangle = \sum_{\alpha \in P} \frac{1-2q_{m,\alpha,j}}{m-1} (1-\rho) + O(\kappa).
        \end{equation}
        \item Finally, suppose $h[j]=1$. Observe that if $j$ is still an input, then the three-term interaction is a random variable with mean zero; but if $d < \alpha < m$, it is a product of four relevant inputs. Therefore,
        \begin{equation} \label{Term-1-h[m]=1-incorrect-same-height}
            \frac{1}{n(m-1)} \sum_{\alpha} \langle \bm{x}_m, \bm{x}_{\alpha}, \bm{x}_j \rangle = \sum_{\alpha = d+1}^{m-1} \frac{1-2q_{m,\alpha,j}}{m-1} (1-\rho) + O(\kappa).
        \end{equation}
    \end{enumerate}
    \item $h[m]>1$. This condition restricts $d + k/2 < m \leq d+k-1$. Again, this case can be divided into three following sub-cases.
    \begin{enumerate}[noitemsep, nolistsep, leftmargin=*]
        \item Suppose $j \in P$. Then the inner product $\langle x_m, x_{\alpha}, x_j \rangle$ is always a product of two, four, or six inputs if $\alpha \in P$, or it is a random variable with mean zero otherwise. Therefore,
        \begin{equation} \label{Term-1-h[m]>1-incorrect-height-zero}
            \frac{1}{n(m-1)} \sum_{\alpha} \langle \bm{x}_m, \bm{x}_{\alpha}, \bm{x}_j \rangle = \sum_{\alpha \in P} \frac{1-2q_{m,\alpha,j}}{m-1} (1-\rho) + O(\kappa).
        \end{equation}
        \item If $p[j] = m$, then clearly $h[j] = h[m] - 1 \geq 1$. If $\alpha = j'$ is another child of $m$, then $\langle x_m, x_{j'}, x_j \rangle = 1 - 2q_{m, j', j}$. On the other hand, if $d < \alpha < m$ and $\alpha \neq j'$, then any inner product is the product of an even number of relevant input variables, so it is a random variable with mean $1-\rho$. In all other cases, the inner product has mean zero. Hence,
        \begin{equation} \label{Term-1-h[m]>1-correct}
            \frac{1}{n(m-1)} \sum_{\alpha} \langle \bm{x}_m, \bm{x}_{\alpha}, \bm{x}_j \rangle = \frac{1-2q_{m, c_1[m], c_2[m]}}{m-1} + \sum_{d < \alpha < m, \alpha \neq j'} \frac{1-2q_{m,\alpha,j}}{m-1} (1-\rho) + O(\kappa).
        \end{equation}
        \item Now suppose $h[j] > 0$ but $p[j] \neq m$. Then as long as $d < \alpha < m$, any inner product is the product of an even number of relevant input variables, so it is a random variable with mean $1-\rho$. In all other cases, the inner product has mean zero. Hence,
        \begin{equation} \label{Term-1-h[m]>1-incorrect-same-height}
            \frac{1}{n(m-1)} \sum_{\alpha} \langle \bm{x}_m, \bm{x}_{\alpha}, \bm{x}_j \rangle = \sum_{\alpha = d+1}^{m-1} \frac{1-2q_{m,\alpha,j}}{m-1} (1-\rho) + O(\kappa).
        \end{equation}
    \end{enumerate}
\end{enumerate}

\subsection{The second term}

We now evaluate the second term.
\begin{align}
    & \frac{1}{n(m-1)} \left \langle -\bm{1}_n + 2\hat{\bm{z}}_m, 2 \hat{\bm{z}}_m, \bm{x}_j - \hat{\bm{z}}_m \right \rangle\\
    & = \frac{1}{n(m-1)} \left \langle -\bm{1}_n, 2 \hat{\bm{z}}_m, \bm{x}_j - \hat{\bm{z}}_m \right \rangle + \frac{1}{n(m-1)} \left \langle 2\hat{\bm{z}}_m, 2d^3 \hat{\bm{z}}_m, \bm{x}_j - \hat{\bm{z}}_m \right \rangle\\
    & = -\frac{2}{n(m-1)} \langle \hat{\bm{z}}_m, x_j \rangle + \frac{2}{n(m-1)} \langle \hat{\bm{z}}_m^2 \rangle + \frac{4}{n(m-1)} \langle \hat{\bm{z}}_m^2, \bm{x}_j \rangle - \frac{4}{n(m-1)} \langle \hat{\bm{z}}_m^3 \rangle.
\end{align}

We focus on the first and third terms in the final expression. In particular, we compute the following:
\begin{equation*}
    \frac{1}{n} \langle \hat{\bm{z}}_m, \bm{x}_j \rangle = \frac{1}{n(m-1)} \sum_{\alpha} \langle \bm{x}_{\alpha}, \bm{x}_j \rangle \quad \& \quad \frac{1}{n} \langle \hat{\bm{z}}_m^2, \bm{x}_j \rangle = \frac{1}{n(m-1)^2} \sum_{\alpha, \beta} \langle \bm{x}_{\alpha}, \bm{x}_{\beta}, \bm{x}_j \rangle.
\end{equation*}

For the first two-order term, we divide into two cases.
\begin{enumerate}[noitemsep, nolistsep, leftmargin=*]
    \item $h[j]=0$. Clearly, $\langle x_j, x_j \rangle = 1$, and if $\alpha \in P \setminus \{ j \}$, the inner product is a product of two relevant bits, so it is a random variable with mean $1-\rho$. Otherwise, it is a variable with mean zero. Hence,
    \begin{equation} \label{Term-2.1-h[j]=0}
        \frac{1}{n} \langle \hat{\bm{z}}_m, \bm{x}_j \rangle = \frac{1}{m-1} + \frac{k-1}{m-1} (1-\rho) + O(\kappa).
    \end{equation}
    \item $h[j] \geq 1$. Be aware that $x_j$ itself is a product of $2^{h[j]}$ relevant input bits. Again, $\langle x_j, x_j \rangle = 1$. If $h[\alpha] = 0$, then the inner product is a product of an odd number of input bits, so has mean zero. For all other cases, i.e. $d < \alpha < m$ and $\alpha \neq j$, the product is the same as a product of an even number of relevant bits, and has mean $1-\rho$ without corruption. Therefore,
    \begin{equation} \label{Term-2.1-h[j]>0}
        \frac{1}{n} \langle \hat{\bm{z}}_m, \bm{x}_j \rangle = \frac{1}{m-1} + \sum_{d < \alpha < m, \alpha \neq j} \frac{1-2q_{\alpha, j}}{m-1} (1-\rho) + O(\kappa).
    \end{equation}
\end{enumerate}

For the next three-order term, we again divide into two cases on the height of $j$.
\begin{enumerate}[noitemsep, nolistsep, leftmargin=*]
    \item $h[j]=0$. Be aware that $x_j$ itself is a product of $2^{h[j]}$ relevant input bits. Therefore, to transform a product $\langle x_{\alpha}, x_{\beta}, x_j \rangle$ to an even multiplication of relevant bits, exactly one of $\alpha$ and $\beta$ must have height at least one, and the other must have height zero. The order can be different, so there are in total $2k(m-d-1)$ possible combinations. Hence, the total sum in this case is
    \begin{equation} \label{Term-2.3-h[j]=0}
        \frac{1}{n} \langle \hat{\bm{z}}_m^2, \bm{x}_j \rangle = \frac{2k(1-2q_{\alpha})}{(m-1)^2}(1-\rho) + O(\kappa).
    \end{equation}
    The summands only need to take care of the poisoning rate of the nodes with non-zero height because input bits are not corrupted.
    \item $h[j] \geq 1$. In this case, observe that $\langle x_{\alpha}, x_{\beta}, x_j \rangle$ has mean $1-\rho$ if and only if $h[\alpha] = h[\beta] = 0$ or $h[\alpha], h[\beta] \geq 1$. So there are in total $k^2 + (m-d-1)^2$ possibilities:
    \begin{equation} \label{Term-2.3-h[j]>0}
        \frac{1}{n} \langle \hat{\bm{z}}_m^2, \bm{x}_j \rangle = \sum_{\alpha, \beta = d+1}^{m-1} \frac{1-2q_{\alpha, \beta, j}}{(m-1)^2}(1-\rho) + \sum_{\alpha, \beta \in P} \frac{1-2q_j}{(m-1)^2} (1-\rho) + O(\kappa).
    \end{equation}
\end{enumerate}

\subsection{Differences of gradient updates}

Given $j \neq j' < m$, we compute the differences of gradient updates for $L$ with respect to $w_{j,m}$ and $w_{j',m}$ as the following:
\begin{align} \label{gradient-difference-general}
    \Delta_{m,j,j'} & = \frac{\partial L}{\partial w_{j,m}} (\mathbf{W}) - \frac{\partial L}{\partial w_{j',m}} (\mathbf{W})\\
    & = -\frac{1}{n(m-1)} \left \langle \bm{x}_m, 2 \hat{\bm{z}}_m, \bm{x}_j - \bm{x}_{j'} \right \rangle + \frac{1}{n(m-1)} \left \langle -\bm{1}_n + 2\hat{\bm{z}}_m, 2d^3 \hat{\bm{z}}_m, \bm{x}_j - \bm{x}_{j'} \right \rangle\\
    & = \left( -\frac{2}{n(m-1)} \left \langle \bm{x}_m, 2 \hat{\bm{z}}_m, \bm{x}_j \right \rangle \right) - \left( -\frac{2}{n(m-1)} \left \langle \bm{x}_m, 2 \hat{\bm{z}}_m, \bm{x}_{j'} \right \rangle \right)\\
    & \quad \quad + \left( \frac{2}{n(m-1)} \left \langle -\bm{1}_n + 2\hat{\bm{z}}_m, \hat{\bm{z}}_m, \bm{x}_j \right \rangle \right) - \left( \frac{2}{n(m-1)} \left \langle -\bm{1}_n + 2\hat{\bm{z}}_m, \hat{\bm{z}}_m, \bm{x}_{j'} \right \rangle \right)\\
    & = \left( - \frac{2}{n(m-1)^2} \sum_{\alpha} \langle \bm{x}_m, \bm{x}_{\alpha}, \bm{x}_j \rangle \right) - \left( - \frac{2}{n(m-1)^2} \sum_{\alpha} \langle \bm{x}_m, \bm{x}_{\alpha}, \bm{x}_{j'} \rangle \right)\\
    & \quad \quad + \left( -\frac{1}{n(m-1)^2} \sum_{\alpha} \langle \bm{x}_{\alpha}, \bm{x}_j \rangle \right) - \left( -\frac{1}{n(m-1)^2} \sum_{\alpha} \langle \bm{x}_{\alpha}, \bm{x}_{j'} \rangle \right)\\
    & \quad \quad + \left( \frac{1}{n(m-1)^3} \sum_{\alpha, \beta} \langle \bm{x}_{\alpha}, \bm{x}_{\beta}, \bm{x}_j \rangle \right) - \left( \frac{1}{n(m-1)^3} \sum_{\alpha, \beta} \langle \bm{x}_{\alpha}, \bm{x}_{\beta}, \bm{x}_{j'} \rangle \right)\\
    & \quad \quad + O(d^{-2-\epsilon/4}).
\end{align}

Observe that the first two terms depend on locations of all $\{ m,j,j' \}$, while last four terms above do not depend on $m$ but only depend on the locations of $j$ and $j'$. So we compute them separately. In particular, for each choice of $m$, we must first compute the ``\textbf{correct}'' gradient $\partial L / \partial w_{c_1[m], m} = \partial L / \partial w_{c_2[m], m}$, and then compute the ``\textbf{incorrect}'' gradients depending on the location of $j$. Concretely, the steps are the following.
\begin{enumerate}[noitemsep, nolistsep, leftmargin=*]
    \item Assume $h[m]=1$, compute $\partial L / \partial w_{c_1[m], m} = \partial L / \partial w_{c_2[m], m}$.
    \begin{enumerate}[noitemsep, nolistsep, leftmargin=*]
        \item Compute the gradient $\partial L / \partial w_{j,m}$ if $h[j]=0$ but $p[j] \neq m$.
        \item Compute the gradient $\partial L / \partial w_{j',m}$ if $h[j']>0$.
        \item Subtract the correct gradient with the previous two incorrect gradients.
    \end{enumerate}
    \item Assume $h[m]>1$, compute $\partial L / \partial w_{c_1[m], m} = \partial L / \partial w_{c_2[m], m}$.
    \begin{enumerate}[noitemsep, nolistsep, leftmargin=*]
        \item Compute the gradient $\partial L / \partial w_{j,m}$ if $h[j]=0$.
        \item Compute the gradient $\partial L / \partial w_{j',m}$ if $h[j']>0$ but $p[j] \neq m$.
        \item Subtract the correct gradient with the previous two incorrect gradients.
    \end{enumerate}
\end{enumerate}

For Step 1, the equation is
\begin{equation}
    \left( -\frac{2}{m-1} \right) \times \cref{Term-1-h[m]=1-correct} + \left( -\frac{2}{m-1} \right) \times \cref{Term-2.1-h[j]=0} + \left( \frac{4}{m-1} \right) \times \cref{Term-2.3-h[j]=0}.
\end{equation}
For Step 1.(a), the equation is
\begin{equation}
    \left( -\frac{2}{m-1} \right) \times \cref{Term-1-h[m]=1-incorrect-lower-height} + \left( -\frac{2}{m-1} \right) \times \cref{Term-2.1-h[j]=0} + \left( \frac{4}{m-1} \right) \times \cref{Term-2.3-h[j]=0}.
\end{equation}
For Step 1.(b), the equation is
\begin{equation}
    \left( -\frac{2}{m-1} \right) \times \cref{Term-1-h[m]=1-incorrect-same-height} + \left( -\frac{2}{m-1} \right) \times \cref{Term-2.1-h[j]>0} + \left( \frac{4}{m-1} \right) \times \cref{Term-2.3-h[j]>0}.
\end{equation}
For Step 1.(c), the differences are
\begin{equation}
    \left( -\frac{2}{m-1} \right) \times \cref{Term-1-h[m]=1-correct} - \left( -\frac{2}{m-1} \right) \times \cref{Term-1-h[m]=1-incorrect-lower-height} = - \frac{2 \rho (1-2q_m)}{(m-1)^2};
\end{equation}
and
\begin{align}
    & \left( -\frac{2}{m-1} \right) \times \cref{Term-1-h[m]=1-correct} - \left( -\frac{2}{m-1} \right) \times \cref{Term-1-h[m]=1-incorrect-same-height}\\
    & \quad + \left( -\frac{2}{m-1} \right) \times \cref{Term-2.1-h[j]=0} - \left( -\frac{2}{m-1} \right) \times \cref{Term-2.1-h[j]>0}\\
    & \quad + \left( \frac{4}{m-1} \right) \times \cref{Term-2.3-h[j]=0} - \left( \frac{4}{m-1} \right) \times \cref{Term-2.3-h[j]>0}.
\end{align}

For Step 2, the equation is
\begin{equation}
    \left( -\frac{2}{m-1} \right) \times \cref{Term-1-h[m]>1-correct} + \left( -\frac{2}{m-1} \right) \times \cref{Term-2.1-h[j]>0} + \left( \frac{4}{m-1} \right) \times \cref{Term-2.3-h[j]>0}.
\end{equation}
For Step 2.(a), the equation is
\begin{equation}
    \left( -\frac{2}{m-1} \right) \times \cref{Term-1-h[m]>1-incorrect-height-zero} + \left( -\frac{2}{m-1} \right) \times \cref{Term-2.1-h[j]=0} + \left( \frac{4}{m-1} \right) \times \cref{Term-2.3-h[j]=0}.
\end{equation}
For Step 2.(b), the equation is
\begin{equation}
    \left( -\frac{2}{m-1} \right) \times \cref{Term-1-h[m]>1-incorrect-same-height} + \left( -\frac{2}{m-1} \right) \times \cref{Term-2.1-h[j]>0} + \left( \frac{4}{m-1} \right) \times \cref{Term-2.3-h[j]>0}.
\end{equation}
For Step 2.(c), the differences are
\begin{align}
    & \left( -\frac{2}{m-1} \right) \times \cref{Term-1-h[m]>1-correct} - \left( -\frac{2}{m-1} \right) \times \cref{Term-1-h[m]>1-incorrect-height-zero}\\
    & \quad + \left( -\frac{2}{m-1} \right) \times \cref{Term-2.1-h[j]>0} - \left( -\frac{2}{m-1} \right) \times \cref{Term-2.1-h[j]=0}\\
    & \quad + \left( \frac{4}{m-1} \right) \times \cref{Term-2.3-h[j]>0} - \left( \frac{4}{m-1} \right) \times \cref{Term-2.3-h[j]=0};
\end{align}
and
\begin{equation}
    \left( -\frac{2}{m-1} \right) \times \cref{Term-1-h[m]>1-correct} - \left( -\frac{2}{m-1} \right) \times \cref{Term-1-h[m]>1-incorrect-same-height} = -\frac{2\rho (1-2q_{m, c_1[m], c_2[m]})}{(m-1)^2}.
\end{equation}

We have computed one gradient difference for each case of $m$, and there is one remaining for each $m$. Observe that, we only need to compute three following expressions:
\begin{align}
    & G_{h[m]=1} (m, j, \rho) = -\frac{2}{m-1} \times \left( \cref{Term-1-h[m]=1-correct} - \cref{Term-1-h[m]=1-incorrect-same-height} \right);\\
    & G_{h[m]>1} (m, j, \rho) = -\frac{2}{m-1} \times \left( \cref{Term-1-h[m]>1-correct} - \cref{Term-1-h[m]>1-incorrect-height-zero} \right);
\end{align}
and
\begin{equation} \label{S(m-j)}
    \left( -\frac{2}{m-1} \right) \times \cref{Term-2.1-h[j]=0} - \left( -\frac{2}{m-1} \right) \times \cref{Term-2.1-h[j]>0} + \left( \frac{4}{m-1} \right) \times \cref{Term-2.3-h[j]=0} - \left( \frac{4}{m-1} \right) \times \cref{Term-2.3-h[j]>0}.
\end{equation}
Observe that the remaining gradients can be equivalently expressed as
\begin{equation}
    G_{h[m]=1} (m, j, \rho) + \cref{S(m-j)} \quad \& \quad G_{h[m]>1} (m, j, \rho) - \cref{S(m-j)}.
\end{equation}

Using the ingredients from earlier results, for $m \in \{ d+1, \ldots, d + k/2 \}$ and $d<j<m$, we have
\begin{small}
\begin{align}
\begin{split}
    G_{h[m]=1} (m, j, \rho) & = \frac{-2(1-2q_m)}{(m-1)^2} + \frac{-2(k-1)(1-2q_m)}{m-1} (1-\rho)\\
    & \quad \quad - \sum_{\alpha = d+1}^{m-1} \frac{-2(1-2q_{m,\alpha,j})}{m-1} (1-\rho) + O(d^{-2-\epsilon/4}).
\end{split}
\end{align}
\end{small}
Similarly, for $m > d+k/2$ and $d<j<m$, we have
\begin{small}
\begin{align}
\begin{split}
    G_{h[m]>1} (m, j, \rho) & = \frac{-2(1-2q_{m, c_1[m], c_2[m]})}{(m-1)^2} + \sum_{d < \alpha < m, \alpha \neq j'} \frac{-2(1-2q_{m,\alpha,j})}{(m-1)^2} (1-\rho)\\
    & \quad \quad - \frac{-2k(1-2q_{m,\alpha,j})}{(m-1)^2} (1-\rho) + O(d^{-2-\epsilon/4}).
\end{split}
\end{align}
\end{small}

Observe that \Cref{Term-2.1-h[j]=0}-(\ref{Term-2.3-h[j]>0}) can all be factored out by $1-\rho$, so we may have
\begin{equation}
    \text{\Cref{S(m-j)}} = (1-\rho) \cdot S(m,j)
\end{equation}
for an expression $S(m,j)$. Using our results of \Cref{Term-2.1-h[j]=0}-(\ref{Term-2.3-h[j]>0}) earlier, we have
\begin{small}
\begin{equation}
    S(m, j) = - \frac{2(k-1)}{(m-1)^2} + \sum_{d < \alpha < m, \alpha \neq j} \frac{2(1-2q_{\alpha})}{(m-1)^2} + \sum_{\alpha = d+1}^{m-1} \frac{8k(1-2q_{\alpha})}{(m-1)^3} -  \sum_{\alpha, \beta = d+1}^{m-1} \frac{4(1-2q_{\alpha, \beta, j})}{(m-1)^3} - \sum_{\alpha, \beta \in P} \frac{4(1-2q_m)}{(m-1)^3}.
\end{equation}
\end{small}

Finally, for each case of $m$, we conclude the differences between correct and incorrect gradients:
\begin{itemize}[noitemsep, nolistsep, leftmargin=*]
    \item $h[m]=1$. Then the differences are
    \begin{equation}
        \left \{ -\frac{2\rho (1-2q_m)}{(m-1)^2} \right \} \quad \& \quad \left \{ G_{h[m]=1} (m, j, \rho) + (1-\rho) S(m, j), \quad d < j < m \right \}.
    \end{equation}
    \item $h[m]>1$. Then the differences are
    \begin{equation}
        \left\{ -\frac{2\rho (1-2q_{m, c_1[m], c_2[m]})}{(m-1)^2} \right \} \quad \& \quad \left \{ G_{h[m] > 1} (m, j, \rho) - (1-\rho) S(m, j), \quad d<j<m \right \}.
    \end{equation}
\end{itemize}

\subsection{Conditions are sufficient}

For any choice of $m$, if the condition in either \Cref{condition-one-height} or (\ref{condition-higher-height}) (depending on the height of $m$) holds, then the condition is equivalent with the fact: There exists a value $\mu > -2 - \epsilon/4$ such that for any $j < m$, the differences between correct and incorrect gradients satisfy the following
\begin{equation}
    \Delta_{m, c_1[m], j}, \Delta_{m, c_2[m], j} < -O(d^{\mu}).
\end{equation}
This means the gap between the correct and incorrect gradients is large. If we pick any $\mu' \in (-\mu, 2+\epsilon/4)$ and choose a learning rate $\eta = \Theta (d^{\mu'})$, then after one gradient update, the difference between weights for children of $m$ and weights for non-children is:
\begin{equation}
    \left \lvert \Delta^{(1)}_{m, c_1[m], j} \right \rvert, \left \lvert \Delta^{(1)}_{m, c_2[m], j} \right \rvert \geq O(d^{-\mu + \mu'}) + O(d^{-2 - \epsilon/4 + \mu'}) = O(d^{-\mu + \mu'}).
\end{equation}
The last equality holds because, by the range of $\mu'$, we must have $-\mu + \mu' > 0$ and $-2 - \epsilon/4 + \mu' < 0$; so the second quantity is dominated by the first one. 

The conditions imply that the incorrect weights are smaller than correct weights, so applying the softmax attention score function, for $j \neq c_1[m], c_2[m]$ we have
\begin{equation}
    \sigma_{j} (\bm{w}^{(1)}_m)  \leq \exp \left( - \left \lvert \Delta^{(1)}_{m, c_1[m], j} \right \rvert \right) \leq \exp \left( - \Theta \left( d^{-\mu + \mu'} \right) \right).
\end{equation}
Softmax scores must sum to 1, we must have
\begin{equation}
    \sigma_{c_1[m]} (\bm{w}^{(1)}_m) + \sigma_{c_2[m]} (\bm{w}^{(1)}_m) \geq 1 - \exp \left( - \Theta \left( d^{-\mu + \mu'} \right) \right).
\end{equation}
Moreover, we observe that in this case, the correct attention scores $\sigma_{c_1[m]} (\bm{w}^{(1)}_m)$ and $\sigma_{c_2[m]} (\bm{w}^{(1)}_m)$ are close enough:
\begin{equation}
    \frac{\sigma_{c_1[m]} (\bm{w}^{(1)}_m)}{\sigma_{c_2[m]} (\bm{w}^{(1)}_m)} = \exp \left( w^{(1)}_{c_1[m],m} - w^{(1)}_{c_2[m],m} \right) \leq \exp \left( O \left( d^{-2 - \epsilon/4 + \mu'} \right) \right) \leq 1 + O \left( d^{-2 - \epsilon/4 + \mu'} \right),
\end{equation}
where the last inequality holds because $e^t \leq 1 + O(t)$ for small $t > 0$. By symmetry, we have the same upper bound for $\sigma_{c_1[m]} (\bm{w}^{(1)}_m) / \sigma_{c_2[m]} (\bm{w}^{(1)}_m)$. As a result, we have
\begin{equation}
    \frac{1}{2} - O \left( d^{-2 - \epsilon/4 + \mu'} \right) \leq \sigma_{c_1[m]} (\bm{w}^{(1)}_m), \sigma_{c_2[m]} (\bm{w}^{(1)}_m) \leq \frac{1}{2} + O \left( d^{-2 - \epsilon/4 + \mu'} \right).
\end{equation}
The equation above implies that for each step $d < m \leq d+k-1$, the attention layer at step $m$ almost computes the average of two children nodes, and all information from other non-children nodes are dominated and essentially vanish as $d$ becomes large.

We now show that using the attention scores, every prediction step, including the final output, has a vanishing loss. Let $\bm{x}$ be a $d$-dimensional binary input vector. For every $d < m \leq d+k-1$, let $\hat{z}^{(1)}_m$ be the empirical output of the attention layer, then $\phi(\hat{z}^{(1)}_m)$ is the empirical prediction, and the prediction loss for step $m$ is
\begin{align}
    \epsilon_m & = \left \lvert \phi \left( \hat{z}^{(1)}_m \right) - \phi \left( \frac{x_{c_1[m]} + x_{c_2[m]}}{2} \right) \right \rvert\\
    & \leq 2 \times \left \lvert \hat{z}^{(1)}_m - \frac{x_{c_1[m]} + x_{c_2[m]}}{2} \right \rvert\\
    & \leq 2 \times \left \lvert \hat{z}^{(1)}_m - \frac{\hat{x}_{c_1[m]} + \hat{x}_{c_2[m]}}{2} \right \rvert + 2 \times \left \lvert \frac{x_{c_1[m]} + x_{c_2[m]}}{2} - \frac{\hat{x}_{c_1[m]} + \hat{x}_{c_2[m]}}{2} \right \rvert \\
    & = 2d \exp \left( - \Theta \left( d^{-\mu + \mu'} \right) \right) + 2 \times \left \lvert \sigma_{c_1[m]} (\bm{w}_m^{(1)}) - \frac{1}{2} \right \rvert + 2 \times \left \lvert \sigma_{c_2[m]} (\bm{w}_m^{(1)}) - \frac{1}{2} \right \rvert + 2 \epsilon_{m-1}\\
    & = O(d^{-2 - \epsilon/4 + \mu'}) + O(\epsilon_{m-1}).
\end{align}
If $m = d+1$, then $\epsilon_{m-1} = \epsilon_{d} = 0$ because the $d$-th value is still an input. Therefore, this upper bound holds for every $d < m \leq d+k-1$ and take $m = d+k-1$ so that $x_m = y$, we conclude that $| \hat{y} - y | = | \hat{x}_{d+k-1} - x_{d+k-1} | = O(d^{-2 - \epsilon/4 + \mu'})$.

\subsection{Conditions are necessary}

Now suppose there is at least one $m$ such that the condition in \Cref{condition-one-height} or (\ref{condition-higher-height}) (depending on the height of $m$) does not hold. This implies that, for this $m$, there exists at least one non-child node $j$ such that $j < m$ and the gap between this incorrect gradient $\partial L / \partial w_{m, j}$ and the correct gradients $\partial L / \partial w_{m, c_1[m]}$, $\partial L / \partial w_{m, c_2[m]}$ is too small to be distinguished. Precisely, there exists a number $\delta \leq -2 - \epsilon/4$ such that
\begin{equation}
    \left \lvert \Delta^{(1)}_{m, c_1[m], j} \right \rvert, \left \lvert \Delta^{(1)}_{m, c_2[m], j} \right \rvert \leq O(d^{\delta}) + O(d^{-2 - \epsilon/4}) = O(d^{-2 - \epsilon/4}) = \left \lvert \Delta^{(1)}_{m, c_1[m], c_2[m]} \right \rvert.
\end{equation}
Using the same analysis in the proof of sufficiency, the small gaps implies that the attention scores $\sigma_{c_1[m]} (\bm{w}^{(1)}_m), \sigma_{c_2[m]} (\bm{w}^{(1)}_m), \sigma_{j} (\bm{w}^{(1)}_m)$ are close. Consider the optimal scenario under this case, that the condition in \Cref{condition-one-height} or (\ref{condition-higher-height}) (depending on the height of $m$) holds for any other $j' \neq j$, then we have $\sigma_{c_1[m]} (\bm{w}^{(1)}_m) +  \sigma_{c_2[m]} (\bm{w}^{(1)}_m) +  \sigma_{j} (\bm{w}^{(1)}_m) = 1 - e^{- \Theta(d^{\nu})}$ for some $\nu > 0$ and for any $a, b \in \{ c_1[m], c_2[m], j \}$, we have
\begin{equation}
    \frac{\sigma_a (\bm{w}^{(1)}_m)}{\sigma_b (\bm{w}^{(1)}_m)} \leq \exp \left( O \left( d^{-2 - \epsilon/4} \right) \right) \leq 1 + O \left( d^{-2 - \epsilon/4} \right).
\end{equation}
Equivalently,
\begin{equation}
    \frac{1}{3} - O \left( d^{-2 - \epsilon/4} \right) \leq \sigma_{c_1[m]} (\bm{w}^{(1)}_m), \sigma_{c_2[m]} (\bm{w}^{(1)}_m), \sigma_{j} (\bm{w}^{(1)}_m) \leq \frac{1}{3} + O \left( d^{-2 - \epsilon/4} \right)
\end{equation}

For a sufficiently large $d$, i.e. as $d \to \infty$, we may regard the predictor at step $m$ is exactly the function $\hat{x}_m = \phi \left( \frac{1}{3} x_{c_1[m]} + \frac{1}{3} x_{c_2[m]} + \frac{1}{3} x_j \right)$, where the ground truth must still be $x_m = \phi \left( \frac{1}{2} x_{c_1[m]} + \frac{1}{2} x_{c_2[m]} \right)$. If the $d$-dimensional inputs are uniformly generated, then with probability exactly 0.5, the sample satisfies the property that $x_{c_1[m]} = -x_{c_2[m]}$, so the true prediction for step $m$ is $x_m = \phi(0) = -1$. On the other hand, the empirical prediction is $\hat{x}_m = \phi(x_j/3)$ and therefore fixed as $\phi(1/3) = \phi(-1/3) = -1/3$. Therefore, the the error for prediction at stepm $m$ is lower bounded as the following:
\begin{equation}
    \mathbb{E}_{\bm{x} \sim \text{Uniform}(\{ \pm 1 \}^d)} \left \lvert \hat{x}_m - x_m \right \rvert \geq \frac{1}{2} \times \left \lvert -1 + \frac{1}{3} \right \rvert = \frac{1}{3} = \Omega (1).
\end{equation}

We conclude this proof by showing that, if one step $m$ for some $d < m \leq d+k-1$ has a non-negligible loss, then the loss for the final prediction also has a non-negligible loss, i.e. $| y - \hat{y} | = | x_{d+k-1} - \hat{x}_{d+k-1} | = \Omega(1)$.

We first show that such an error of node $m$ causes a non-negligible damage on its parent node $p[m]$, i.e. $| x_{p[m]} - \hat{x}_{p[m]} | = \Omega(1)$. Denote $m'$ as the unique sibling of $m$, i.e. $p[m] = p[m']$. Then clearly $x_{p[m]} = x_m x_{m'}$. Hence, the following inequality satisfies; we abuse the notation by using $\mathbb{E}$ as the uniform generation.
\begin{align}
    \mathbb{E}_{\bm{x} \sim \text{Uniform}(\{ \pm 1 \}^d)} | x_{p[m]} - \hat{x}_{p[m]} | & = \mathbb{E} | x_m x_{m'} - \hat{x}_m \hat{x}_{m'} |\\
    & \stackrel{d \to \infty}{=} \mathbb{E} | x_m x_{m'} - x_m \hat{x}_{m'} |\\
    & = \mathbb{E} \left [ |x_m| \cdot |x_{m'} - \hat{x}_{m'}| \right]\\
    & = \mathbb{E} |x_{m'} - \hat{x}_{m'}| = \Omega(1).
\end{align}
The inequality holds for every $m$ regardless of its location, and inductively, this non-negligible error propagates to higher ancestors of $m$, and ultimately to the root of the tree.

Recall that this is the ``best'' scenario when the condition in \Cref{condition-one-height} or (\ref{condition-higher-height}) fails for $m$, i.e. only one non-child node $j$ has a prohibitively high attention score. If more non-children nodes fail, say $f$ of them in the set $F \subseteq [m-1]$, then $\hat{z}_m = \sum_{f \in F} x_f /(f-1)$. By Hoeffding's inequality, we have
\begin{equation}
    \mathbb{P} \left( \phi \left( \frac{\sum_{f \in F} x_f}{f-1} \right) < 0 \right) = \mathbb{P} \left( \frac{\sum_{f \in F} x_f}{f-1} < \frac{1}{2} \right) \geq 1 - e^{-\Omega(f)}.
\end{equation}
Nevertheless, the true prediction is still uniform, so
\begin{equation}
    \mathbb{E} | x_m - \hat{x}_m | \geq \frac{1}{2} \left( 1 - e^{-\Omega(f)} \right) = \frac{1}{2} - o(1) = \Omega(1).
\end{equation}
Using the same argument for the root prediction for the case above, the final prediction also suffers an error of $\Omega(1)$, as desired.

\section{Experiment details}

The feedforward layer function for the transformer is the same function in Section 3.2, i.e.
\begin{equation*}
    \phi(x) = \begin{cases}
    d^3 x^2 + d^{-3} -1, \quad x \in (-d^{-3}, d^{-3});\\
    2|x| - 1, \quad \quad \quad \quad \text{otherwise}.
    \end{cases}
\end{equation*}

The testing data are uniformly generated. Although for general tasks, it is natural to expect that the test error should be large if the training and testing distributions are different. However, in this $k$-parity test, if the training is successful, the predictor is expected to identify the positions of relevant bits, regardless the training distribution where it was learned. Therefore, testing the predictor for all values of $\rho \in \{ 0, 0.25, 0.5, 0.75, 1 \}$ is the only fair measure of the predictor's performance. 

The experiments were ran for five times with $d = 128$ and $k = 64$. The mean and variance values for each case is illustrated in Figure 6, which has the same format as \Cref{fig:heatmap}. For each grid, the mean and variance are computed by the following  standard formulas:
\begin{equation*}
    \text{Mean}(\mu) = \frac{\sum_{i=1}^n x_i}{n} \quad \& \quad \text{Variance} (\sigma^2) = \frac{\sum_{i=1}^n (x_i - \mu)^2}{n}.
\end{equation*}

\clearpage
\vfill
\begin{figure} \label{heat-map-mean-var}
    \centering
    \includegraphics[width = 1.0\textwidth]{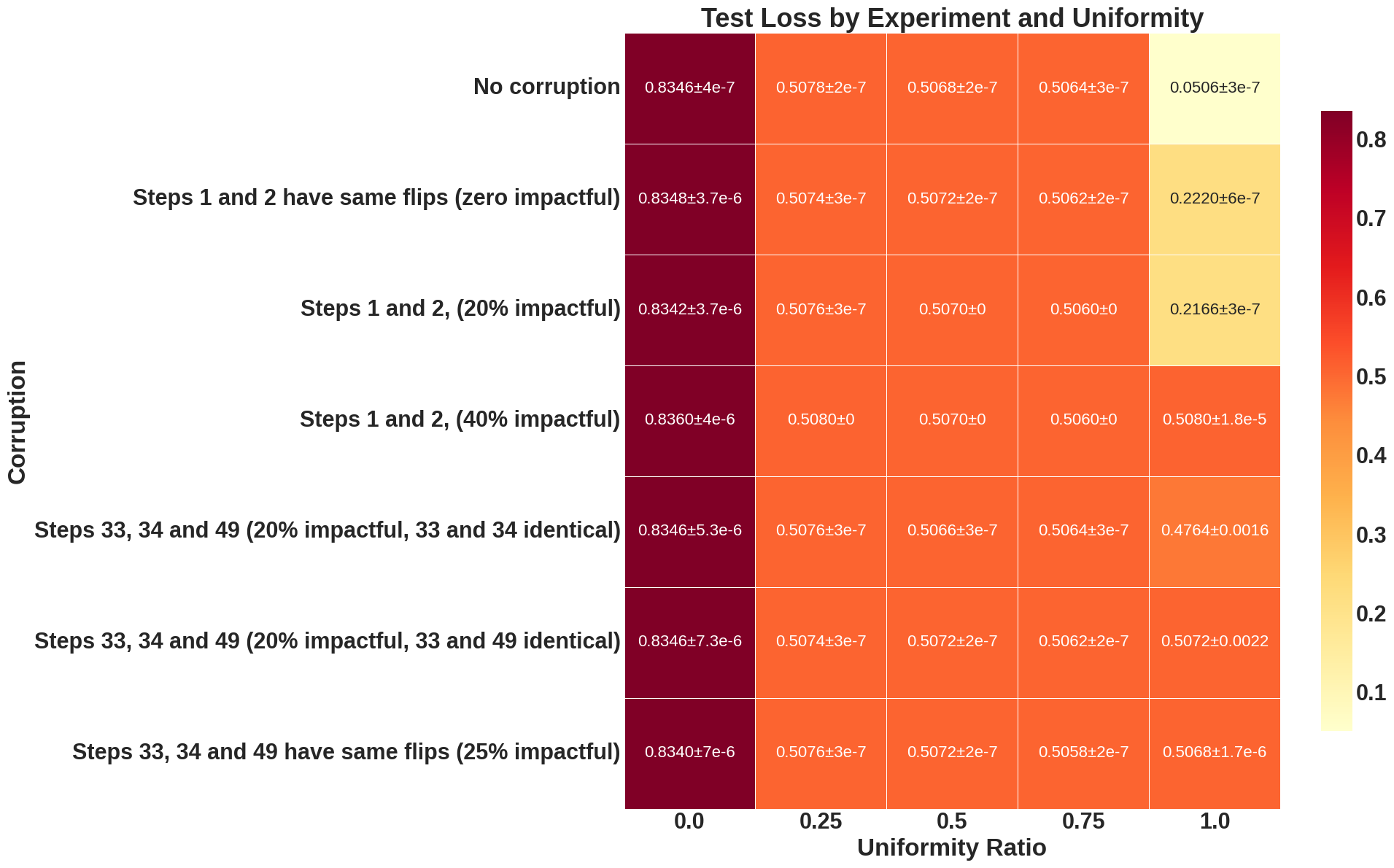}
    \caption{Mean and variance of the test losses after repetitions}
    \label{fig:enter-label}
\end{figure}
\vfill
\clearpage

\end{document}